\documentclass[11pt]{article}
\usepackage[thmeqnumber]{matus}
\setlist[enumerate]{leftmargin=2em}
\usepackage{algorithm}
\usepackage{algorithmic}
\setlist[itemize]{leftmargin=2em}

\usepackage{caption}
\def\barV{\overline V}

\def\baru{\overline u}
\def\barv{\overline v}

\def\btheta{\overline{\theta}}

\def\tv{\widetilde v}
\def\ta{\widetilde a}
\def\talpha{\widetilde \alpha}
\def\tb{\widetilde b}
\def\tu{\widetilde u}
\def\tp{\widetilde p}
\def\sconv{\textup{sconv}}

\makeatletter
\DeclareFontFamily{OMX}{MnSymbolE}{}
\DeclareSymbolFont{MnLargeSymbols}{OMX}{MnSymbolE}{m}{n}
\SetSymbolFont{MnLargeSymbols}{bold}{OMX}{MnSymbolE}{b}{n}
\DeclareFontShape{OMX}{MnSymbolE}{m}{n}{
    <-6>  MnSymbolE5
   <6-7>  MnSymbolE6
   <7-8>  MnSymbolE7
   <8-9>  MnSymbolE8
   <9-10> MnSymbolE9
  <10-12> MnSymbolE10
  <12->   MnSymbolE12
}{}
\DeclareFontShape{OMX}{MnSymbolE}{b}{n}{
    <-6>  MnSymbolE-Bold5
   <6-7>  MnSymbolE-Bold6
   <7-8>  MnSymbolE-Bold7
   <8-9>  MnSymbolE-Bold8
   <9-10> MnSymbolE-Bold9
  <10-12> MnSymbolE-Bold10
  <12->   MnSymbolE-Bold12
}{}

\let\llangle\@undefined
\let\rrangle\@undefined
\DeclareMathDelimiter{\llangle}{\mathopen}{MnLargeSymbols}{'164}{MnLargeSymbols}{'164}
\DeclareMathDelimiter{\rrangle}{\mathclose}{MnLargeSymbols}{'171}{MnLargeSymbols}{'171}
\makeatother

\renewcommand{\Pr}{\textsc{Pr}}

\def\hcR{\widehat{\cR}}

\def\sgn{\textup{sgn}}

\def\cs{\bar{\partial}}
\def\hs{\hat{\partial}}

\def\tgamma{\widetilde{\gamma}}
\def\ngamma{\mathring{\gamma}}

\def\hgamma{\widehat{\gamma}}
\def\gntk{\gamma_{\scriptscriptstyle{\textup{ntk}}}}
\def\ggl{\gamma_{\scriptscriptstyle{\textup{gl}}}}
\def\gnc{\gamma_{\scriptscriptstyle{\textup{nc}}}}
\def\bara{\bar a}
\def\baru{\bar u}

\def\barW{\overline W}
\def\hatW{\widehat W}

\def\sgn{\textup{sgn}}

\title{\textbf{Feature selection with gradient descent\\{}on two-layer networks in low-rotation regimes}}

\author{Matus Telgarsky\thanks{\href{mailto:mjt@illinois.edu}{\texttt{<mjt@illinois.edu>}};
  comments greatly appreciated.}}

\date{}

\begin{document}

\maketitle

\begin{abstract}
  This work establishes low test error of gradient flow (GF)
  and stochastic gradient descent (SGD)
  on two-layer ReLU networks
  with standard initialization,
  in three regimes where key sets of weights rotate little
  (either naturally due to GF and SGD, or due to an artificial constraint),
  and making use of margins as the core analytic technique.
  The first regime is near initialization, specifically until the weights have
  moved by $\mathcal{O}(\sqrt m)$, where $m$ denotes the network width,
  which is in sharp contrast to the $\mathcal{O}(1)$ weight motion allowed by the
  Neural Tangent Kernel (NTK);
  here it is shown that GF and SGD only need a network width and number of samples
  inversely proportional to the NTK margin,
  and moreover that GF attains at least  the NTK margin itself,
  which suffices to establish escape from bad KKT points of the margin objective,
  whereas prior work could only establish nondecreasing but arbitrarily small margins.
  The second regime is the Neural Collapse (NC) setting,
  where data lies in extremely-well-separated groups,
  and the sample complexity scales with the number of groups;
  here the contribution over prior work is an analysis of the
  entire GF trajectory from initialization.
  Lastly, if the inner layer weights are constrained to change in norm only and 
  can not rotate,
  then GF with large widths achieves globally maximal margins, and its sample
  complexity scales with their inverse;
  this is in contrast to prior work,
  which required infinite width and a tricky dual convergence assumption.
  As purely technical contributions, this work
  develops a variety of potential functions and other tools
  which will hopefully aid future work.
\end{abstract}

\section{Introduction}

This work studies standard descent methods on two-layer networks of width $m$,
specifically stochastic gradient descent (SGD) and gradient flow (GF) on the logistic
and exponential losses, with a goal of establishing good test test error (low sample
complexity) via margin theory (see \Cref{sec:notation} for full details on the setup).
The analysis considers three settings where weights rotate little, but rather their
magnitudes change a great deal,
whereby the networks \emph{select} or emphasize good features.

The motivation and context for this analysis is as follows.
While the standard promise of deep learning is to provide automatic feature learning,
by contrast, standard optimization-based analyses typically utilize the
Neural Tangent Kernel (NTK) \citep{jacot_ntk},
which suffices to establish low training error
\citep{du_deep_opt,allen_deep_opt,zou_deep_opt},
and even low testing error \citep{arora_2_gen,li_liang_nips,ziwei_ntk},
but ultimately the NTK is equivalent to a linear predictor over \emph{fixed} features given
by a Taylor expansion, and fails to achieve the aforementioned feature learning promise.

Due to this gap, extensive mathematical effort has gone into the study of
\emph{feature learning},
where the feature maps utilized in deep learning (e.g., those given by Taylor expansions)
may change drastically, as occurs in practice.
The promise here is huge:
even simple tasks such as \emph{$2$-sparse parity} (learning the parity of two bits
in $d$ dimensions)
require $d^2/\eps$ samples for a test error of $\eps>0$ within the NTK or any other
kernel regime, but it is possible for ReLU networks \emph{beyond} the NTK regime
to achieve an improved sample complexity of $d/\eps$ \citep{wei_reg}.
This point will be discussed in detail throughout this introduction,
though a key point is that prior work generally changes the algorithm to achieve good
sample complexity;
as a brief summary, many analyses
either add noise to the training process \citep{yingyu_feature,wei_reg},
or train the first layer for only one step and thereafter train only the second
\citep{daniely2020learning,abbe2022merged,barak2022hidden},
and lastly a few others idealize the network and make strong assumptions
to show that global margin maximization occurs, giving the desired $d/\eps$ sample
complexity \citep{chizat_bach_imp}.  As will be expanded upon briefly,
and indeed is depicted in \Cref{fig:2xor} and summarized in \Cref{table:2xor},
in the case of the aforementioned 2-sparse parity,
the present
work will achieve the optimal kernel sample complexity $d^2/\eps$ with a standard
SGD (cf. \Cref{fact:sgd}), and a beyond-kernel sample complexity of $d/\eps$
with an inefficient and somewhat simplified GF (cf. \Cref{fact:gl}).

The technical approach in the present work is to build upon the rich
theory of \emph{margins} in machine learning \citep{boser_svm,schapire_freund_book_final}.
While the classical theory provides
that descent methods on linear models can maximize margins
\citep{zhang_yu_boosting,mjt_margins,nati_logistic}, recent work in deep networks has
revealed that GF eventually converges monotonically to local optima
of the
margin function \citep{kaifeng_jian_margin,dir_align}, and even that, under a variety
of conditions, margins may be \emph{globally} maximized \citep{chizat_bach_imp}.
As mentioned above, global margin maximization is sufficient to beat the NTK sample
complexity for many well-studied problems, including $2$-sparse parity
\citep{wei_reg}.

The contributions of this work fall into roughly two categories.

\begin{enumerate}[font=\bfseries]
  \item 
    \textbf{\Cref{sec:ntk}: Margins at least as good as the NTK.}
    The first set of results may sound a bit unambitious, but already constitute a variety of
    improvements over prior work.  Throughout these contributions, let $\gntk$ denote
    the \emph{NTK margin}, a quantity defined and discussed in detail in
    \Cref{sec:margins}.
    \begin{enumerate}
      \item
        \Cref{fact:sgd}: $\tcO(1/(\eps\gntk^2))$ steps of SGD
        (with one example per iteration) on a network of width
        $m = \widetilde{\Omega}(1/\gntk^2)$
        suffice to achieve test error $\eps>0$.
        For $2$-sparse parity, this suffices to achieve the optimal within-kernel
        sample complexity $d^2/\eps$,
        and in fact the computational cost and sample complexity
        improve upon all existing prior work on efficient methods
        (cf. \Cref{table:2xor}).

      \item
        \Cref{fact:gf:margins}: with similar width and samples, GF also achieves
        test error $\eps>0$.  Arguably more importantly, this analysis gives the
        first guarantee in a general setting that constant margins
        (in fact $\gntk/4096$)
        are achieved; prior work only established nondecreasing but arbitrarily small
        margins \citep{kaifeng_jian_margin}.  Furthermore, this result suffices to
        imply that GF can escape bad local optima of the margin objective
        (cf. \Cref{fact:gf:kkt}).

      \item
        These proofs are not within the NTK: the NTK requires weights to move at most
        $\cO(1)$, whereas weights can move by $\cO(\sqrt{m})$ within these proofs;
        in fact, the first gradient step of SGD is shown to have norm at least
        $\gamma\sqrt{m}$ (and the step size is $\cO(1)$).  These proofs will
        therefore hopefully serve as the basis of other proofs outside the NTK in 
        future work.

    \end{enumerate}

  \item
    \textbf{\Cref{sec:beyond}: Margins surpassing the NTK.}
    The remaining margin results are able to beat the preceding NTK margins,
    however they pay a large price: the network width is exponentially large,
    and the method is GF, not the discrete-time algorithm SGD. Even so, the results
    already require new proof techniques, and again hopefully form the basis for 
    improvements in future work.
    \begin{enumerate}
      \item
        \Cref{fact:nc}: the first result is 
        for Neural Collapse (NC), a setting where data lies in $k$ tightly
        packed clusters which are extremely far apart; in fact, each cluster is correctly labeled
        by some vector $\beta_k$,
        and no other data may fall in the halfspace defined by $\beta_k$.
        Despite this, it is an interesting setting with a variety of empirical and theoretical
        works, for instance showing various asymptotic stability properties
        \citep{papyan2020prevalence}; the contribution here is to 
        analyze the entire GF trajectory from initialization, and show that it achieves
        a margin (and sample complexity) at least as good as the sparse ReLU network with
        one ReLU pointing at each cluster.

      \item
        \Cref{fact:gl}: under a further idealization that the inner weights are constrained
        to change in norm only (meaning the inner weights can not rotate),
        global margin maximization occurs.
        As mentioned before, this suffices to establish sample complexity $d/\eps$ in
        2-sparse parity, which beats the $d^2/\eps$ of kernel methods.
        As a brief comparison to prior work, similar idealized networks were proposed
        by \citet{woodworth2020kernel}, however a full trajectory analysis and relationship
        to global margin maximization under no further assumptions was left open.
        Additionally, without the no-rotation constraint,
        \citet{chizat_bach_imp} showed that infinite-width networks under a tricky
        \emph{dual convergence} assumption also globally maximize margins;
        the proof technique here is rooted in a desire to drop this dual convergence
        assumption, a point which will be discussed extensively in \Cref{sec:beyond}.

      \item
        New potential functions: both \Cref{fact:nc} and \Cref{fact:gl} are proved via
        new potential arguments which hopefully aid future work.
    \end{enumerate}

\end{enumerate}

This introduction will close with
further related work (cf. \Cref{sec:related}),
notation (cf. \Cref{sec:notation}),
and detailed definitions and estimates of the various margin notions
(cf. \Cref{sec:margins}).
Margin maximization at least as good as the NTK is presented in \Cref{sec:ntk},
margins beyond the NTK are in \Cref{sec:beyond}, and open problems and concluding
remarks appear in \Cref{sec:open}.  Proofs appear in the appendices.

\begin{figure}[b!]
\begin{tcolorbox}[enhanced jigsaw, empty, sharp corners, colback=white,borderline north = {1pt}{0pt}{black},borderline south = {1pt}{0pt}{black},left=0pt,right=0pt,boxsep=0pt,rightrule=0pt,leftrule=0pt]
\centering
    \begin{subfigure}[t]{0.49\textwidth}
      \centering
\includegraphics[width=\textwidth]{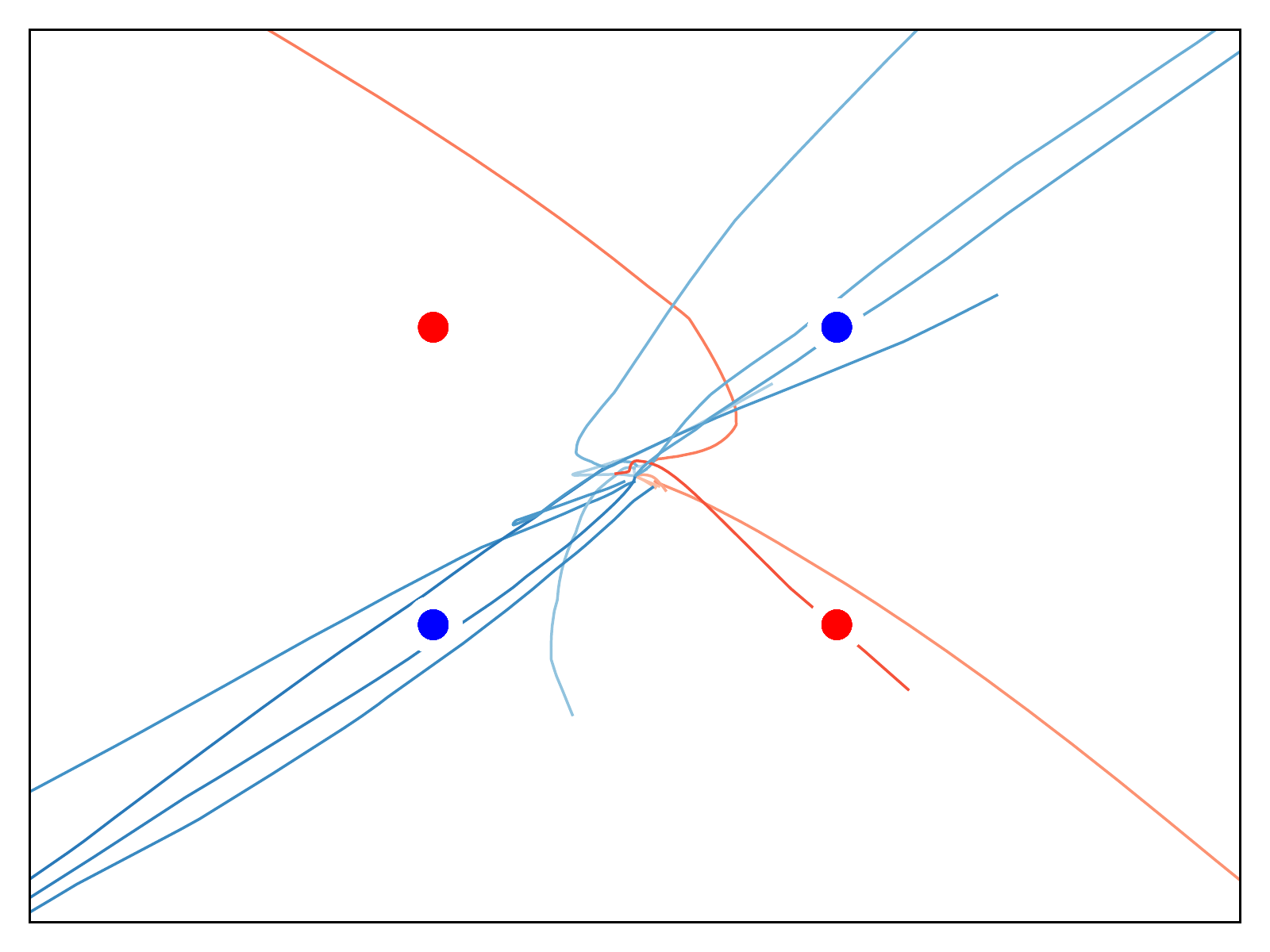}
      \caption{Trajectories $(|a_j|v_j)_{j=1}^m$ with $m=16$ across time.}
      \label{fig:2xor:1}
    \end{subfigure}\hfill
    \begin{subfigure}[t]{0.49\textwidth}
      \centering
\includegraphics[width=\textwidth]{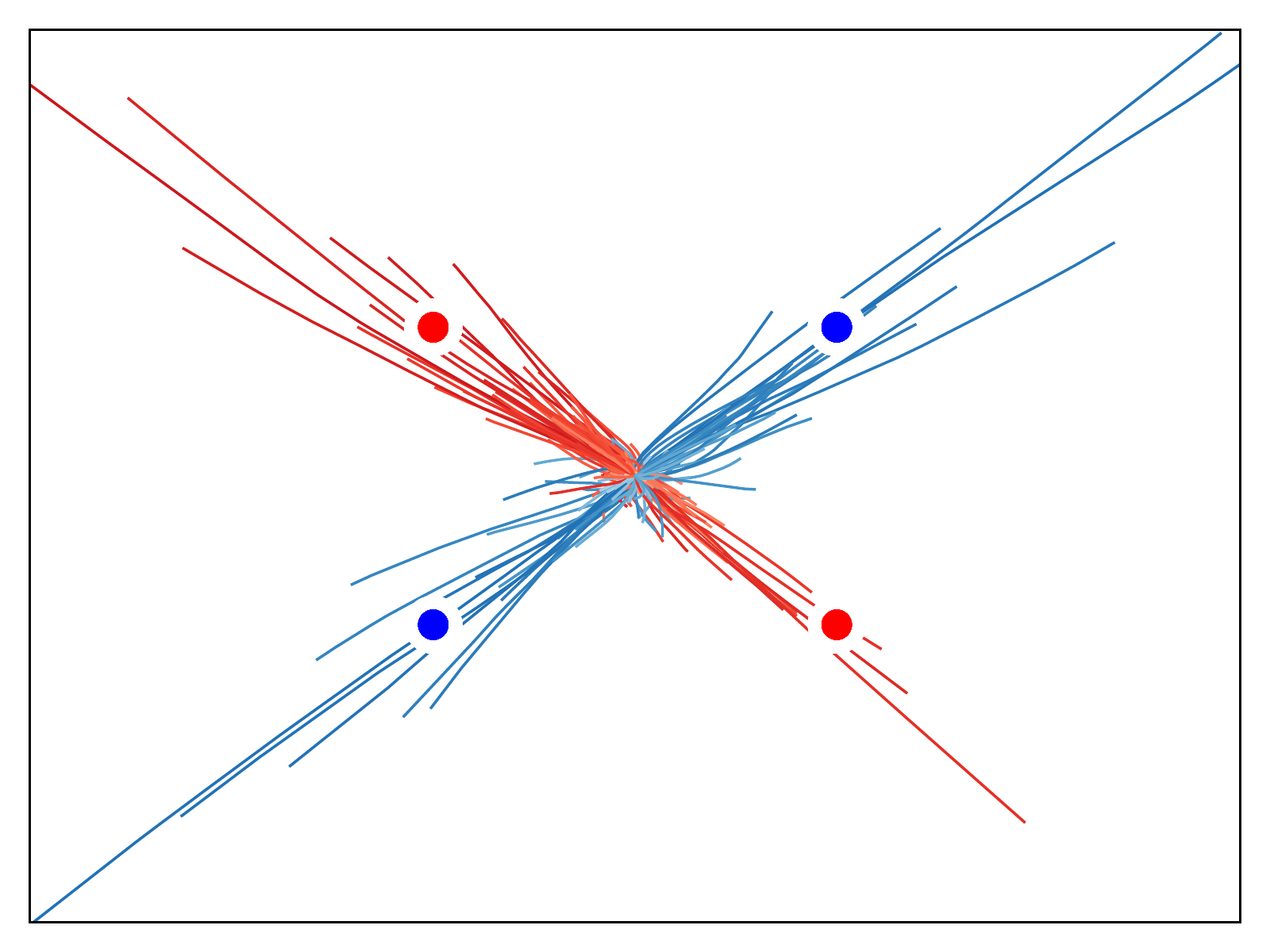}
      \caption{Trajectories $(|a_j|v_j)_{j=1}^m$ with $m=256$ across time.}
      \label{fig:2xor:2}
    \end{subfigure}

    \begin{subfigure}[t]{0.49\textwidth}
      \centering
      \includegraphics[width=\textwidth]{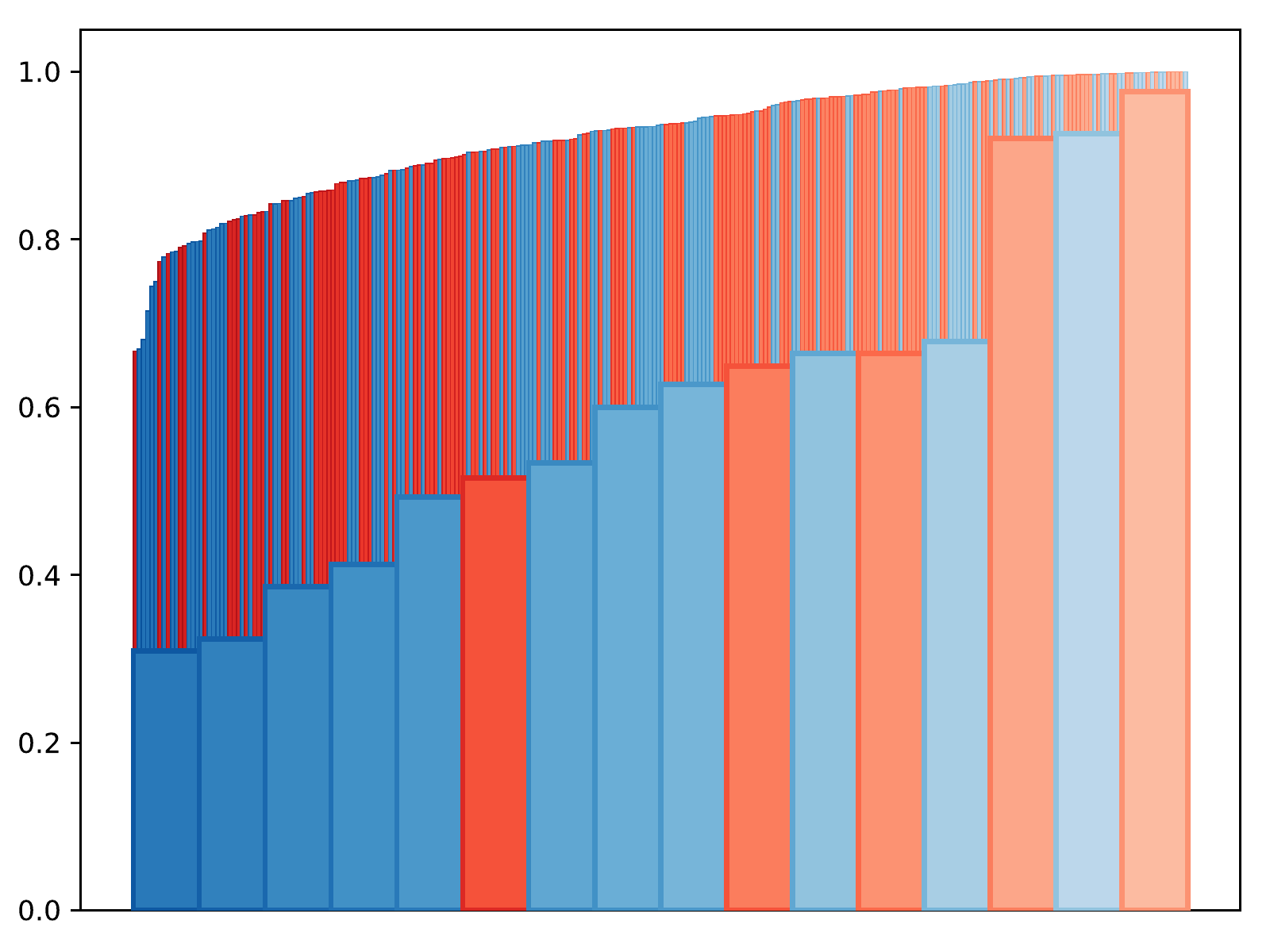}
      \caption{Sorted per-node rotations
for $m\in\{16,256\}$.}
      \label{fig:2xor:3}
    \end{subfigure}\hfill
    \begin{subfigure}[t]{0.49\textwidth}
      \centering
      \includegraphics[width=\textwidth]{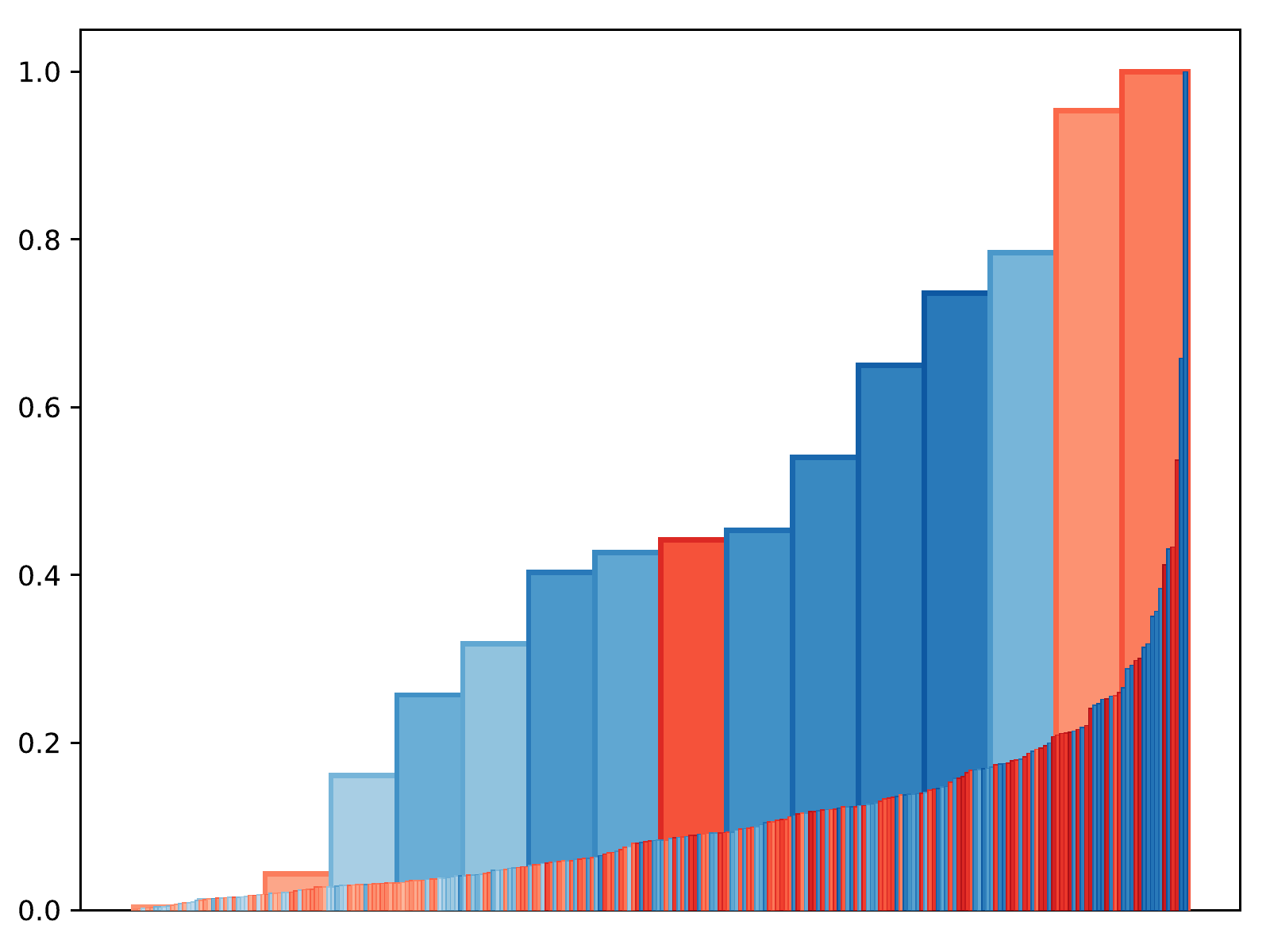}
      \caption{Sorted per-node norms
for $m\in\{16,256\}$.}
      \label{fig:2xor:4}
    \end{subfigure}\caption{Two runs of approximate GF (GD with small step size) on 2-sparse parity
      with $d=20$ and $n=64$ and $m\in\{16,256\}$:
      specifically, a new data point $x$ is sampled uniformly from the hypercube corners
      $\{\pm 1/\sqrt{d}\}^d$, and the label $y$ for simplicity is the parity of the first
      two bits, meaning $y = d x_1 x_2$.  The first two dimensions of the
      data and the trajectories of the $m$ nodes
      are depicted in \Cref{fig:2xor:1,fig:2xor:2} for $m\in\{16,256\}$.
      \Cref{fig:2xor:3} shows per-node rotation over time in sorter order,
      meaning
$\del[2]{\ip{\frac{v_j(0)}{\|v_j(0)\|}}{\frac{v_j(t)}{\|v_j(t)\|}}}_{j=1}^m$.
      \Cref{fig:2xor:4} shows per-node relative norms in sorted order,
      meaning $\del[2]{\frac {\|a_jv_j\|}{\max_k\|a_kv_k\|}}_{j=1}^m$.
      Due to projection onto the first two coordinates,
      the $64$ data points land in $4$ clusters, and are colored red if
      negatively labeled, blue if positively labeled.  Individual nodes
are colored red or blue based on the
      sign of their output weights.  The shades of red and blue darken for nodes whose
      total norm is larger. Comparing \Cref{fig:2xor:3} and \Cref{fig:2xor:4}, large norm and large rotation
      go together.  While \Cref{fig:2xor:1} is highly noisy,
      the behavior of
      \Cref{fig:2xor:2,fig:2xor:3,fig:2xor:4} was highly regular across training.
      The trend of larger width leading to less rotation and greater norm imbalance
      was consistent across runs.  This somewhat justifies
      the lack of rotation with exponentially large width in \Cref{fact:gl},
      and the overall terminology choice of feature \emph{selection}.
}
    \label{fig:2xor}
   \end{tcolorbox}
\end{figure}

\subsection{Further related work}\label{sec:related}

\paragraph{Margin maximization.}
The concept and analytical use of margins in machine learning originated in the
classical perceptron convergence analysis of \citet{novikoff}.
The SGD analysis in \Cref{fact:sgd}, as well as the training error analysis
in \Cref{fact:gf} were both established with a variant of the perceptron proof;
similar perceptron-based proofs appeared before \citep{ziwei_ntk,chen_much},
however they required width $1/\gntk^8$, unlike the $1/\gntk^2$ here, and moreover the proofs
themselves were in the NTK regime, whereas the proof here is not.

Works focusing on the \emph{implicit margin maximization} or \emph{implicit bias} of
descent methods are more recent.  Early works on the coordinate descent side are
\citep{boosting_margin,zhang_yu_boosting,mjt_margins}; the proof here of
\Cref{fact:gf:large_margin} uses roughly the proof scheme in \citep{mjt_margins}.
More recently, margin maximization properties of gradient descent were established,
first showing global margin maximization in linear models
\citep{nati_logistic,riskparam_logreg}, then showing nondecreasing \emph{smoothed}
margins of general homogeneous networks (including multi-layer ReLU networks)
\citep{kaifeng_jian_margin}, and the aforementioned \emph{global} margin maximization
result for 2-layer networks under dual convergence and infinite width
\citep{chizat_bach_imp}.  The potential functions used here in \Cref{fact:gl,fact:nc}
use ideas from \citep{nati_logistic,kaifeng_jian_margin,chizat_bach_imp}, but also
the shallow linear and deep linear proofs of \citet{refined_pd,align}.

\paragraph{Feature learning.}
There are many works in feature learning; a few also carrying explicit
guarantees on 2-sparse parity are summarized in \Cref{table:2xor}.
An early work with high relevance to the present work is \citep{wei_reg},
which in addition to establishing that the NTK requires $\Omega(d^2/\eps)$ samples whereas
$\cO(d/\eps)$ suffice for the global maximum margin solution,
also provided a noisy Wasserstein Flow (WF) analysis which achieved the maximum margin
solution, albeit using noise,
infinite width, and continuous time
to aid in local search,
The global maximum margin work of \citet{chizat_bach_imp} was mentioned before,
and will be discussed in \Cref{sec:beyond}.
The work of \citet{barak2022hidden} uses a two phase algorithm: the first step
has a large minibatch and effectively learns the support of the parity in an unsupervised
manner, and thereafter
only the second layer is trained, a convex problem which is able to identify the signs
within the parity; as in \Cref{table:2xor}, this work stands alone in terms of the narrow
width it can handle.
The work of \citep{abbe2022merged} uses a similar two-phase approach, and while it
can not learn precisely the parity, it can learn an interesting class of
``staircase'' functions, and presents many valuable proof techniques.
Another work which operates in two phases and can learn an interesting class
of functions which excludes the parity (specifically due to a Jacobian condition)
is the recent work of \citep{damian2022neural}.
Other interesting feature learning works are \citep{yingyu_feature,bai2019beyond}.

\begin{table}
   \begin{tcolorbox}[enhanced jigsaw, empty, sharp corners, colback=white,borderline north = {1pt}{0pt}{black},borderline south = {1pt}{0pt}{black},left=0pt,right=0pt,boxsep=0pt,rightrule=0pt,leftrule=0pt]
  \centering
\begin{tabular}{l|l|l||lll}
    Reference
    &
    Algorithm
    &
    Technique
&
    $m$
    &
    $n$
    &
    $t$
    \\
    \hline
    \hline
    \citep{ziwei_ntk}&
    SGD
    &
    perceptron
&
$d^{8}$ 
    & $d^2/\eps$  
    & $d^2/\eps$
\\
    \Cref{fact:sgd}
     & SGD
     & perceptron
&   $d^2$ &     $d^2/\eps$  &   $d^2/\eps$
\\
    \citep{barak2022hidden}
     & 2-phase SGD
     &
     correlation
& $\cO(1)$ &  $d^4/\eps^2$  & $d^2/\eps^2$  \\
    \citep{wei_reg}
    &
    WF+noise
    &
    margin
&       $\infty$ &  $d/\eps$  &     $\infty$\
\\
    \citep{chizat_bach_imp}
    & WF
    &
    margin
&     $\infty$ &  $d/\eps$  &     $\infty$
\\
    \Cref{fact:gl}
      & scalar GF 
      &
      margin
&   $d^d$  &    $d/\eps$  &     $\infty$
\end{tabular}
  \caption{Performance on 2-sparse parity by a variety of works, loosely organized
    by technique; see \Cref{sec:related} for details.
    Briefly, $m$ denotes width, $n$ denotes
    \emph{total} number of samples (across all iterations), and $t$ denotes the number
    of algorithm iterations.
    Overall, the table captures tradeoffs in all parameters,
    and understanding the Pareto frontier is an interesting avenue for future work.
  }
  \label{table:2xor}
   \end{tcolorbox}
\end{table}

\subsection{Notation}\label{sec:notation}

\paragraph{Architecture and initialization.}
With the exception of \Cref{fact:gl},
the architecture will be a $2$-layer ReLU network of the form
$x\mapsto F(x;w) = \sum_j a_j \sigma(v_j^\T x) = a^\T \sigma(Vx)$,
where $\sigma(z) = \max\{0,z\}$ is the ReLU,
and where $a\in\R^m$ and $V\in\R^{m\times d}$ have initialization roughly matching the
variances of \texttt{pytorch} default initialization,
meaning $a\sim \cN_m/\sqrt{m}$ ($m$ iid Gaussians with variance $1/m$)
and $V \sim \cN_{m\times d}/\sqrt{d}$ ($m\times d$ iid Gaussians with variance $1/d$).
These parameters $(a,V)$ will be collected into a tuple
$W = (a,V)\in \R^{m}\times \R^{m\times d}\equiv \R^{m\times (d+1)}$,
and for convenience per-node tuples
$w_j = (a_j,v_j)\in\R\times\R^d$ will often be used as well.

Given a pair $(x,y)$ with $x\in\R^d$ and $y\in\{\pm 1\}$,
the \emph{prediction} or \emph{unnormalized margin mapping}
is $p(x,y;W) = y F(x;W) = y a^\T \sigma(Vx)$;
when examples $((x_i,y_i))_{i=1}^n$ are available, a simplified notation
$p_i(W) := p(x_i,y_i;W)$ is often used, and moreover define a single-node
variant $p_i(w_j) := y_i a_j \sigma(v_j^\T x_i)$.
Throughout this work, $\|x\|\leq 1$.

It will often be useful to consider normalized parameters within proofs:
define $\tv_j := v_j / \|v_j\|$ and $\ta_j := \sgn(a_j) := a_j/|a_j|$.

\paragraph{SGD and GF.}
The loss function $\ell$ in this work will always be either
the exponential loss $\lexp(z) := \exp(-z)$,
or the logistic loss $\llog(z) := \ln(1+\exp(-z))$;
the corresponding empirical risk $\hcR$ is
\[
  \hcR(p(W)) := \frac 1 n \sum_{i=1}^n \ell(p_i(W)),
\]
which used $p(W) := (p_1(W),\ldots,p_n(W))\in\R^n$.
With $\ell$ and $\hcR$ in hand, the descent methods are defined as
\begin{align}
  W_{i+1}
  &:= W_i - \eta \hs_W \ell(p_i(W_i)),
  &\textup{stochastic gradient descent (SGD),}
  \label{eq:sgd}
  \\
  \dot W_t
  &:=
  \frac \dif {\dif t} W_t
=
  - \cs_W\hcR(p(W_t)),
  &\textup{gradient flow (GF),}
  \label{eq:gf}
\end{align}
where $\hs$ and $\cs$ are appropriate generalizations of subgradients for the 
present nonsmooth nonconvex setting, detailed as follows.
For SGD,
$\hs$
will denote any valid element of the Clarke differential
(i.e., a measurable selection);
for example,
$\hs F(x;W) = \del{\sigma(Vx), \sum_j a_j \sigma'(v_j^\T x) e_j x^\T}$,
where $e_j$ denotes the $j$th standard basis vector,
and $\sigma'(v_j^\T x_i)\in[0,1]$ is chosen in some consistent and measurable way.
This approach is able to model what happens in a standard software library such as
\texttt{pytorch}.  For GF, $\cs$ will denote the unique minimum norm element
of the Clarke differential; typically, GF is defined as a differential inclusion,
which agrees with this minimum norm Clarke flow almost everywhere, but here the minimum
norm element is used to define the flow purely for fun.
Since at this point there is a vast
array of ReLU network literature using Clarke differentials, it is merely asserted here
that chain rules and other basic properties work as required, and $\hs$ and $\cs$ are
used essentially as gradients, and details are deferred
to the detailed discussions in \citep{kaifeng_jian_margin,dir_align,lyu2021gradient}.

Many expressions will have a fixed time index $t$ (e.g., $W_t$), but to reduce clutter, it will either
appear as a subscript, or as a function (e.g., $W(t)$),
or even simply dropped entirely.

\paragraph{Margins and dual variables.}
Firstly, for convenience, often $\ell_i(W) = \ell(p_i(W))$ and $\ell'_i(W) := \ell'(p_i(W))$
are used; since $\ell'_i$ is negative, often $|\ell'_i|$ is written.

The start of developing margins and dual variables
is the observation that $F$ and $p_i$ are \emph{2-homogeneous} in $W$, meaning $F(x; cW)
= ca^\T \sigma(cVx) = c^2 F(x;w)$ for any $x\in\R^d$ and $c\geq 0$
(and $p_i(cW) = c^2 p_i(W)$).  It follows that
$F(x;W) = \|W\|^2 F(x;W/\|W\|)$, and thus $F$ and $p_i$ scale quadratically in $\|W\|$,
and it makes sense to define a \emph{normalized} prediction mapping $\tp_i$ and
margin $\gamma$ as
\[
  \tp_i(W) := \frac {p_i(W)}{\|W\|^2},
  \qquad
  \gamma(W) := \frac{ \min_i p_i(W)}{\|W\|^2} = \min_i \tp_i(W).
\]
Due to nonsmoothness, $\gamma$ can be hard to work with, thus,
following \citet{kaifeng_jian_margin},
define the \emph{smoothed margin $\tgamma$} and the \emph{normalized smoothed margin
$\ngamma$} as
\[
  \tgamma(W) := \ell^{-1}\del{n\hcR(W)} = \ell^{-1}\del{\sum_i \ell(p_i(W))},
  \qquad
  \ngamma(W) := \frac {\tgamma(W)}{\|W\|^2},
\]
where a key result is that $\ngamma$ is eventually nondecreasing \citep{kaifeng_jian_margin}.
These quantities may look complicated and abstract, but note for $\lexp$ that
$\tgamma(W) := - \ln\del{ \sum_i \exp(-p_i(W)) }$.

When working with gradients of losses and of smoothed margins, it will be convenient
to define \emph{dual variables $(q_i)_{i=1}^n$}
\[
  q := \nabla_p \ell^{-1}\del{\sum_i \ell(p_i)}
  = \frac {\nabla_p \sum_i \ell(p_i)}{\ell'(\ell^{-1}(\sum_i \ell(p_i)))}
  = \frac {\nabla_p \sum_i \ell(p)}{\ell'(\tgamma(p))},
\]
which made use of the inverse function theorem. Correspondingly
define $\cQ := \ell'(\tgamma(p))$, whereby $-\ell'_i = q_i \cQ$;
for the exponential loss, $\cQ = \sum_i \exp(-p_i)$ and $\sum_i q_i =1$,
and while these quantities
are more complicated for the logistic loss, they eventually satisfy
$\sum_i q_i \geq 1$
\citep[Lemma 5.4, first part, which does not depend on linear predictors]{refined_pd}.

\paragraph{Margin assumptions.}\label{sec:margins}
It is easiest to first state the \emph{global} margin definition, used in
\Cref{sec:beyond}.  The first version is stated on a finite sample.

\begin{assumption}\label{ass:gl:M-F}
  For given examples $((x_i,y_i))_{i=1}^n$, there exists a scalar $\ggl > 0$
  and parameters $((\alpha_k,\beta_k))_{k=1}^r$ with $\alpha_j \in \R$
  with $\|\alpha\|_1=1$ and $\beta_k\in\R^d$ with $\|\beta_k\|_2=1$,
  where
  \[
    \min_i y_i \sum_{k=1}^r \alpha_k \sigma(\beta_k^\T x_i) \geq \ggl.
    \qedhere
  \]
\end{assumption}

This definition can also be extended to hold over a distribution, whereby
it holds for all finite samples almost surely.

\begin{assumption}\label{ass:gl:M-D}
  There exists $\ggl>0$ and parameters $((\alpha_k,\beta_k))_{k=1}^r$
  so that \Cref{ass:gl:M-F} holds almost surely
  for any $((x_i,y_i))_{i=1}^n$ drawn iid for any $n$ from the underlying distribution
  (with the same $\ggl$ and $((\alpha_k,\beta_k))_{k=1}^r$).
\end{assumption}

\Cref{ass:gl:M-F,ass:gl:M-D} are used in this work to approximate the best possible
margin amongst all networks of any width; in particular, the formalism here is a
simplification of the max-min characterization given in \citep[Proposition 12,
optimality conditions]{chizat_bach_imp}.  As will be seen shortly in
\Cref{fact:margins:parity}, this definition
suffices to achieve sample complexity $d/\eps$ with gradient flow on $2$-sparse parity,
as in
\Cref{table:2xor}.  Lastly, while the appearance of an $\ell_1$ norm may be a surprise,
it can be seen to naturally arise from $2$-homogeneity:
\[
  \frac{p_i(W)}{\|W\|^2}
= \frac {\sum_j p_i(w_j)}{\|W\|^2}
  = \sum_j \alpha_j \tp_i(w_j),
\qquad
  \quad
  \textup{where }
  \alpha_j := \frac {\|w_j\|^2}{\|W\|^2},
  \textup{ thus }
  \|\alpha\|_1 = 1.
\]

Next comes the definition of the \emph{NTK margin $\gntk$}.
As a consequence of $2$-homogeneity, $\ip{W}{\cs_W p_i(W)} = 2p_i(W)$,
which can be interpreted as a linear predictor with weights $W$ and features
$\cs_W p_i(W) /2$.
Decoupling the weights and features gives $\ip{W}{\cs_W p_i(W_0)}$,
where $W_0$ is at initialization, and $W$ is some other choice.  To get the full
definition from here, $W$ is replaced with an infinite-width object via
expectations;
similar definitions originated
with the work of \citet{nitanda_refined}, and were then used in
\citep{ziwei_ntk,chen_much}.

\begin{assumption}\label{ass:M-F}
  For given examples $((x_i,y_i))_{i=1}^n$, there exists a scalar $\gntk > 0$
  and a mapping $\theta : \R^{d+1} \to \R^{d+1}$ with $\theta(w) = 0$ whenever $\|w\|\geq 2$
  so that
  \[
    \min_i
    \bbE_{w \sim \cN_\theta}
    \ip{\theta(w)}{\cs_{w} p_i(w) } \geq \gntk,
  \]
  where $w = (a,v)\sim \cN_\theta$ means $a\sim \cN$ and $v\sim\cN_d/\sqrt{d}$.
\end{assumption}

Similarly to the global maximum margin definition, \Cref{ass:M-F} can be restated
over the distribution.

\begin{assumption}\label{ass:M-D}
  There exists $\gntk>0$ and a transport $\theta :\R^{d+1} \to \R^{d+1}$
  so that \Cref{ass:M-F} holds almost surely
  for any $((x_i,y_i))_{i=1}^n$ drawn iid for any $n$ from the underlying distribution
  (with the same $\gntk$ and $\theta$).
\end{assumption}

Before closing this section, here are a few estimates of $\gntk$ and $\ggl$.
Firstly, both function classes are universal approximators, and thus the assumption
can be made to work for any prediction problem with pure conditional probabilities
\citep{ntk_apx}.  Next, as a warmup, note the following estimates of $\gntk$ and $\ggl$,
for linear predictors, with an added estimate of showing the value of working with
both layers in the definition of $\gntk$.

\begin{proposition}\label{fact:gamma:linear}.
  Let examples $((x_i,y_i))_{i=1}^n$ be given, and suppose they are linearly separable:
  there exists $\|\baru\| = 1$ and $\hgamma > 0$ with $\min_i y_i x_i^\T \baru \geq \hgamma$.
  \begin{enumerate}
    \item
      Choosing $\theta(a, v) := \del{ 0, \sgn(a)\baru}\cdot \1 [\|(a,v)\|\leq 2]$,
      then \Cref{ass:M-F} holds with $\gntk \geq \frac {\hgamma}{32}$.
    \item
      Choosing $\theta(a, v) := \del{ \sgn(\baru^\T v), 0 }\cdot \1[\|(a,v)\|\leq 2]$,
      then \Cref{ass:M-F} holds with $\gntk \geq \frac{\hgamma}{16\sqrt d}$.
    \item
      Choosing $\alpha = (1/2, -1/2)$ and $\beta = (\baru, -\baru)$,
      then \Cref{ass:gl:M-F} holds with $\ggl \geq \frac {\hgamma}{2}$.
  \end{enumerate}
\end{proposition}

Margin estimates for 2-sparse parity are as follows; the key is that 
$\gntk$ scales with $1/d$ whereas $\ggl$ scales with $1/\sqrt{d}$, which
suffices to yield the separations in \Cref{table:2xor}.  The bound on $\gntk$ is
also necessarily an upper bound, since otherwise the estimates due to
\citet{ziwei_ntk}, which are within the NTK regime, would beat NTK lower bounds
\citep{wei_reg}.

\begin{proposition}\label{fact:margins:parity}
  Suppose 2-sparse parity data,
  meaning inputs are supported on $H_d := \{\pm 1/\sqrt{d}\}^d$,
  and for any $x\in H_d$, the label is the product of two fixed coordinates
  $d x_a x_b$ with $a\neq b$.
  \begin{enumerate}
    \item
      \Cref{ass:M-D} holds with $\gntk \geq \frac {1}{50d}$.
    \item
      \Cref{ass:gl:M-D} holds with $\ggl \geq \frac {1}{\sqrt{8d}}$.
  \end{enumerate}
\end{proposition}

\section{Margins at least as good as the NTK}\label{sec:ntk}

This section collects results which depend on the NTK margin $\gntk$
(cf. \Cref{ass:M-F,ass:M-D}).  SGD is presented first in \Cref{sec:sgd}, with
GF following in \Cref{sec:gf}.  The SGD results will not establish large margins, only
low test error, whereas the GF proofs establish both.  As mentioned before,
these results yield the good computation and sample complexity for $2$-sparse parity
in \Cref{table:2xor},
and are also enough to establish escape from bad KKT points in \Cref{sec:gf}.

\subsection{Stochastic gradient descent}\label{sec:sgd}

The only SGD guarantee in this work is as follows.

\begin{theorem}\label{fact:sgd}
  Suppose the data distribution satisfies \Cref{ass:M-D} for some $\gntk>0$,
  let time $t$ be given,
  and suppose width $m$ and step size $\eta$ satisfy
  \[
    m
    \geq \del{\frac {64\ln(t/\delta)}{\gntk}}^2,
    \qquad
    \eta \in\sbr{ \frac {\gntk}{10\sqrt{m}}, \  \frac {\gntk^2}{6400} }
    .
  \]
  Then, with probability at least $1-7\delta$, the SGD iterates $(W_s)_{s\leq t}$ with logistic loss $\ell = \llog$
  satisfy
  \begin{align*}
    \min_{s<t} \Pr\sbr{ p(x,y;W_s) \leq 0 }
    &\leq \frac {8 \ln(1/\delta)}{t} + \frac {2560}{t \gntk^2},
    &\textup{(test error bound),}
    \\
    \max_{s<t} \|W_s - W_0\|& \leq \frac {80\eta \sqrt{m}}{\gntk},
                            &\textup{(norm bound).}
  \end{align*}
\end{theorem}

Note that while $\max_{s<t} \|W_s-W_0\| \leq 80 \eta \sqrt{m}/\gntk$ is only
an upper bound, in fact, by \Cref{fact:gaussians:margin},
the first gradient has norm $\gntk \sqrt{m}$, thus one step (with maximal step size
$\eta = \gntk^2/6400$)
is enough to exit the NTK regime.  As another incidental remark, note that this proof
requires the logistic loss $\llog$, and in fact breaks with the exponential loss $\lexp$.
Lastly, for the 2-sparse parity, $\gntk \geq 1 / (50d)$ as in \Cref{fact:margins:parity},
which after plugging in to \Cref{fact:sgd} gives the corresponding row of \Cref{table:2xor}.

As discussed previously, the width is only $1/\gntk^2$, whereas prior work has
$1/\gntk^8$
\citep{ziwei_ntk,chen_much}.
The proof of \Cref{fact:sgd} is in \Cref{sec:app:sgd}, but here is a sketch.

\begin{enumerate}\item
    \emph{Sampling a good finite-width comparator $\barW$.}
    Perhaps the heart of the proof is showing that the parameters
    $\btheta \in\R^{m\times (d+1)}$
    given by $\btheta_j := \theta(w_j)$
    satisfy
    \[
      \ip{\btheta}{\hs p_i(W_0)} \geq \frac {\gntk\sqrt{m}}{2},
      \qquad \forall i.
    \]
    More elaborate versions of this are used in the proof, and sometimes require a bit
    of surprising algebra (cf. \Cref{fact:gaussians:margin}), which for instance
    seem to be able to treat $\sigma$ as though it were smooth, and without the usual
    careful activation-accounting in ReLU NTK proofs.

  \item
    \emph{Standard expand-the-square.}  As is common in optimization proofs, the core potential
    is a squared Euclidean norm to a good comparator (in this case
    denoted by $\barW$, and defined in terms of $W_0$ and $\btheta$),
    whereby expanding the square gives
    \begin{align*}
      \|W_{s+1} - \barW\|^2
      &=
      \|W_s - \eta \hs \ell_s(w) - \barW\|^2
      \\
      &=
      \|W_s -  \barW\|^2
      - 2\eta \ip{\hs\ell_s(W_s)}{W_s - \barW}
      + \eta^2 \enVert{\hs \ell_s(W_s)}^2
      \\
      &=
      \|W_s -  \barW\|^2
      + 2\eta \ell_s'(W_s)\ip{\hs p_s(W_s)}{\barW - W_s}
      + \eta^2 \ell'_s(W_s)^2 \enVert{\hs p_s(W_s)}^2
      .
    \end{align*}
    Applying $\sum_s$ to both sides and telescoping, the (summations of the) last two terms
    will need to be controlled.  A worrisome prospect is $\|\hs p_s(W_s)\|^2$;
    by the form of $p_s$ and the large initial weight norm, this can be expected to scale
    as $\cO(m)$, but $\eta^2 = \cO(1)$;
    how can this term be swallowed?  This point will be returned to shortly.

    Another critical aspect of the proof is that, in order to control various terms,
    it will be necessary to maintain $\max_{s<t}\|W_s - W_0\|= \cO(\sqrt{m})$ throughout.
    The proof handles this in a way which
    is common in deep network optimization proofs (albeit with a vastly larger norm here):
    by carefully choosing the parameters of the
    proof, one can let $\tau$ denote the first iteration the bound is violated,
    and then derive that $\tau > t$,
    and all the derivations with the assumed bound in fact hold unconditionally.

    The middle term $2\eta \ell_s'(W_s)\ip{\hs p_s(W_s)}{\barW - W_s}$ will be handled
    by a similar trick to one used in \citep{ziwei_ntk}:
    convexity can still be applied to $\ell_s$
    (just not to $\ell$ composed with $p$), and the remaining expression
    can be massaged via homogeneity and the choice of $\barW$.

  \item
    \emph{Controlling $\ell'_s(W_s)^2 \|\hs p_s(W_s)\|^2$: the perceptron argument.}
    Using the first point above, it is possible to show
    \[
      \|W_t - W_0\|
      \geq \frac 1 2 \ip{-\btheta}{W_t - W_0}
      \geq
      \eta \sum_{s<t} |\ell'_s(W_s)| \frac {\gamma \sqrt{m}}{4},
    \]
    which can then be massaged to 
    control the aforementioned squared gradient term.
    Interestingly, this proof step is reminiscent of the perceptron convergence proof,
    and the quantity $\sum_{s<t} |\ell'_s(W_s)|$ is exactly the analog of the mistake
    bound quantity central in perceptron proofs \citep{novikoff}.
    Indeed, this proof never ends up caring
    about the loss terms, and derives the final test error bound (a zero-one loss!)
    via this perceptron term.
\end{enumerate}

\subsection{Gradient flow}\label{sec:gf}

Whereas SGD gave a test error guarantee for free, producing a 
comparable test error bound with GF in this section will require much more work.
This section also sketches the main steps of the proof, and then closes
with a discussion of escaping bad KKT points.

\begin{theorem}
  \label{fact:gf:margins}
  Suppose the data distribution satisfies \Cref{ass:M-D} for some $\gntk>0$,
  and the GF curve $(W_s)_{s\geq 0}$ uses $\ell\in\{\lexp,\llog\}$
  on an architecture whose width $m$ satisfies
  \[
    m \geq \del{\frac {640 \ln(n/\delta)}{\gntk}}^2.
  \]
  Then, with probability at least $1-15\delta$,
there exists $t$ with $\|W_t-W_0\|= \gntk \sqrt{m} / 32$, and
  \begin{align*}
    \ngamma(W_s) \geq \frac {\gntk^2}{4096}
    \qquad\textup{and}\qquad
    \Pr[p(x,y;W_s) \leq 0]
    \leq
    \cO\del{
      \frac {\ln(n)^3}{n\gntk^4}
      + \frac {\ln \frac 1 \delta}{n}
    }
    &&\forall s\geq t,
  \end{align*}
  and moreover the specified iterate $W_t$ satisfies an improved bound
  \[
    \Pr\sbr{ p(x,y;W_t) \leq 0 }
    \leq
    \cO\del{
      \frac {\ln(n)^3}{n\gntk^2}
      + \frac {\ln \frac 1 \delta}{n}
    }.
  \]
\end{theorem}

A few brief remarks are as follows.  Firstly, for $W_t$, the sample complexity matches
the SGD sample complexity in \Cref{fact:sgd}, though via a much more complicated proof.
Secondly, this proof can handle $\{\llog,\lexp\}$ and not just $\llog$.  Lastly, an odd
point is that a bit of algebra grants a better generalization bound at iterate $W_t$,
but it is not clear if this improved bound holds for all time $s\geq t$; in particular
it is not immediately clear that the nondecreasing margin property established by
\citet{kaifeng_jian_margin} can be applied.

One interesting comparison is to a
leaky ReLU convergence
analysis on a restricted form of linearly separable data due to \citet{lyu2021gradient}.
That work, through an extremely technical and impressive analysis, establishes convergence
to a solution which is equivalent to the best linear predictor.  By contrast, while the work
here does not recover that analysis, due to $\gntk$ being a constant multiple of 
the linear margin (cf. \Cref{fact:gamma:linear}),
the sample complexity is within a constant factor of the best linear predictor,
thus giving a sample complexity comparable to that of \citep{lyu2021gradient} via
a simpler proof in a more general setting.

To prove \Cref{fact:gf:margins}, the first step is essentially the same as the
proof of \Cref{fact:sgd}, however it yields only a training error guarantee, not a 
test error guarantee.

\begin{lemma}
  \label{fact:gf}
Suppose the data distribution satisfies \Cref{ass:M-D} for some $\gntk>0$,
  let time $t$ be given,
  and suppose width $m$ satisfies
  \[
    m
    \geq \del{ \frac {640 \ln( t/\delta)}{\gntk}}^2
    .
  \]
  Then, with probability at least $1-7\delta$,
  the GF curve $(W_s)_{s\in [0, t]}$
  on empirical risk $\hcR$ with loss $\ell\in\{\llog,\lexp\}$
  satisfies
  \begin{align*}
    \hcR(W_t)
    &\leq \frac {1}{5t},
    &\textup{(training error bound),}
    \\
    \sup_{s<t} \|W_s - W_0\|&\leq \frac {\gntk \sqrt{m}}{80},
                            &\textup{(norm bound).}
  \end{align*}
\end{lemma}

Note that this bound is morally equivalent to the SGD bound in \Cref{fact:sgd}
after accounting for the $\gntk^2$ ``units'' arising from the step size.

The second step of the proof of \Cref{fact:gf:margins} is an explicit margin guarantee,
which is missing from the SGD analysis.

\begin{lemma}
  \label{fact:gf:large_margin}
  Let data $((x_i,y_i))_{i=1}^n$ be given satisfying \Cref{ass:M-F} with margin
  $\gntk>0$, and let $(W_s)_{s\geq 0}$ denote the GF curve
  resulting from loss $\ell\in\{\llog,\lexp\}$.
  Suppose the width $m$ satisfies
  \[
    m
    \geq \frac{256 \ln(n/\delta)}{\gntk^2},
  \]
  fix a distance parameter $R := \gntk \sqrt{m}/32$,
  and let time $\tau$ be given so that $\|W_\tau - W_0\|\leq R/2$
  and $\cR(W_\tau) < \ell(0) / n$.
  Then, with probability at least $1-7\delta$,
  there exists a time $t$ with $\|W_t - W_0\| = R$
  so that for all $s\geq t$,
  \[
    \|W_s - W_0\|\geq R
    \qquad\textup{and}\qquad
    \ngamma(W_s) \geq \frac {\gntk^2}{4096},
  \]
  and moreover the rebalanced iterate $\hatW_t := (a_t / \sqrt{\gntk}, V_t\sqrt{\gntk})$
  satisfies $p(x,y;W_t) = p(x,y;\hatW_t)$ for all $(x,y)$, and
  \[
    \ngamma(\hatW_t) \geq \frac {\gntk}{4096}.
  \]
\end{lemma}

Before discussing the proof, a few remarks are in order.  Firstly, the final large margin
iterate $W_t$ is stated as explicitly achieving some distance from initialization;
needing such a claim is unsurprising, as the
margin definition requires a lot of motion in a good direction to clear the noise in $W_0$.
In particular, it is unsurprising
that moving $\cO(\sqrt{m})$ is needed to achieve a good margin, given that the initial weight
norm is $\cO(\sqrt{m})$; analogously, it is not surprising that \Cref{fact:gf} can not be used
to produce a meaningful lower bound on $\ngamma(W_\tau)$ directly.

\begin{figure}[t]
   \begin{tcolorbox}[enhanced jigsaw, empty, sharp corners, colback=white,borderline north = {1pt}{0pt}{black},borderline south = {1pt}{0pt}{black},left=0pt,right=0pt,boxsep=0pt,rightrule=0pt,leftrule=0pt]
  \centering
  \includegraphics[width = 0.4\textwidth]{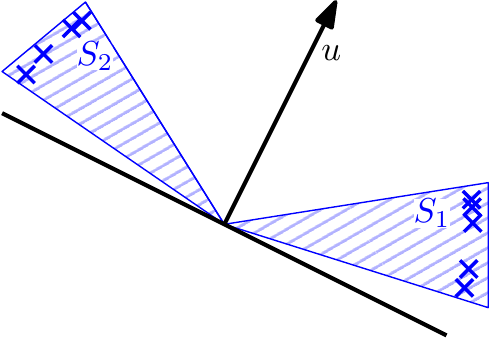}
  \caption{An arrangement of positively labeled points (the blue x's)
    where the margin objective has multiple KKT points, and
    the gradient flow is able to avoid certain bad ones.
    Specifically, as the two cones of data $S_1$ and $S_2$ are rotated away from each other,
    the linear predictor $u$ may still achieve a positive margin, but it will become arbitrarily
    small.  By contrast, pointing two ReLUs at each of $S_1$ and $S_2$
    achieves a much better margin.
    \Cref{fact:gf} is strong enough to establish this occurs, at least for some arrangements of
    the cones, as detailed in \Cref{fact:gf:kkt}.
    This construction is reminiscent of other bad KKT constructions in the literature
    \citep{lyu2021gradient,vardi2021margin}.}
  \label{fig:kkt}
\end{tcolorbox}
\end{figure}

Regarding the proof, surprisingly it can almost verbatim follow a proof scheme originally
designed for margin rates of coordinate descent \citep{mjt_margins}.  Specifically,
noting that $\cs_w \gamma(W_s)$ and $\dot W_s$ are colinear,
the fundamental theorem of calculus (adapted to Clarke differentials) gives
\[
  \gamma(W_t) - \gamma(W_\tau) = \int_\tau^t \frac {\dif}{\dif s}\gamma(W_s)\dif s
  = \int_\tau^t \ip{\cs_W \gamma(W_s)}{\dot W_s} \dif s
  = \int_\tau^t \|\cs_W \gamma(W_s)\| \cdot \|\dot W_s\| \dif s,
\]
and now the terms can be controlled separately.
Assuming the exponential loss for simplicity and recalling the dual variable notation
from \Cref{sec:notation},
then $\cs_W \gamma(W_s) =
\sum_i q_i \cs_W p_i(W_s)$.
Consequently, using the same good property of $\btheta$ discussed in \Cref{sec:sgd}
gives
\[
  \|\cs_W\gamma(W_s)\|
  = \|\sum_i q_i \cs_W p_i(W_s)\|
  \geq \sum_i q_i \ip{\frac {\btheta}{\|\btheta\|}}{\cs_W p_i(W_s)}
  \geq
  \frac {\gntk\sqrt m}{4}.
\]
This leaves the other term of the integral, which is even easier:
\[
  \int_\tau^t \|\dot W_s\| \dif s
  \geq
  \enVert[2]{\int_\tau^t \dot W_s \dif s}
  = \|W_t - W_\tau\|
  \geq \frac R 2,
\]
which completes the proof for the exponential loss.
For the logistic loss, the corresponding
elements $q_i$ do not form a probability vector, and necessitate the use of a $2$-phase
analysis which warm-starts with \Cref{fact:gf}.

To finish the proof of \Cref{fact:gf:margins},
it remains to relate margins to test error, which in order to
scale with $1/n$ rather than $1/\sqrt{n}$ makes use of a beautiful refined margin-based
Rademacher complexity bound due to
\citet[Theorem 5]{nati_smoothness},
and thereafter uses the special structure of $2$-homogeneity to treat the network as an
$\ell_1$-bounded linear combination of nodes, and thereby achieve no dependence,
even logarithmic, on the width $m$.

\begin{lemma}\label{fact:rad}
  With probability at least $1-\delta$ over the draw of $((x_i,y_i))_{i=1}^n$,
  for every width $m$, 
  every choice of weights $W\in\R^{m\times (d+1)}$ with $\ngamma(W)>0$
  satisfies
  \[
    \Pr[p(x,y;W) \leq 0]
    \leq
    \cO\del{
      \frac {\ln(n)^3}{n\ngamma(W)^2}
      + \frac {\ln \frac 1 \delta}{n}
    }.
  \]
\end{lemma}

Combining the preceding pieces yields the proof of \Cref{fact:gf:margins}.
To conclude this section, note that these margin guarantees suffice to establish
that GF can escape bad KKT points of the margin objective.
The construction appears in \Cref{fig:kkt} and is elementary, detailed as follows.
Consider data, all of the same label, lying in two narrow cones $S_1$ and $S_2$.
If $S_1$ and $S_2$ are close together, the global maximum margin network corresponds
to a single linear predictor.  As the angle between $S_1$ and $S_2$ is increased,
eventually the global maximum margin network chooses two separate ReLUs, one pointing
towards each cone; meanwhile, before the angle becomes too large, if $S_1$ and $S_2$ are
sufficiently narrow, there exists a situation whereby a single linear predictor still
has positive margin, but worse than the 2 ReLU solution, and is still a KKT point.

\begin{proposition}\label{fact:gf:kkt}
  Let $((x_i,y_i))_{i=1}^n$ be given as in \Cref{fig:kkt},
  where $y_i = +1$,
  and $(x_i)_{i=1}^n$ all have $\|x_i\|=1$,
  and are partitioned into two sets, $S_1$ and $S_2$,
  such that $\max\{ \ip{x_i}{x_j} : x_i \in S_1, x_j \in S_2\} \leq -1/\sqrt{2}$.
  Then there always exists a margin parameter $\hgamma>0$ and a single new data point
  $x'$ so that the resulting data $S_1 \cup S_2 \cup \{x'\}$
  satisfies the following conditions.
  \begin{enumerate}
    \item
      The maximum margin linear predictor $u$ achieves margin $\hgamma$,
      and is also a KKT point for any shallow ReLU network of width $m\geq 1$.
    \item
      There exists $m_0$ so that for any width $m\geq m_0$,
      GF achieves $\lim_t \ngamma(W_t)>  \hgamma$.
  \end{enumerate}
\end{proposition}

\section{Margins beyond the NTK}\label{sec:beyond}

This section develops two families of bounds beyond the NTK, meaning in particular
that the final margin and sample complexity bounds depend on $\ggl$, rather than
$\gntk$ as in \Cref{sec:ntk}.  On the downside, these bounds all require
exponentially large width, GF, and moreover \Cref{fact:gl} forces the inner layer
to never rotate.  These results are proved with $\lexp$ for convenience, though the
same techniques handling $\llog$ with GF in \Cref{sec:gf} should also work here.

\subsection{Neural collapse (NC)}\label{sec:nc}

The NC setting has data in groups which are well-separated \citep{papyan2020prevalence};
in particular,
data is partitioned into cones, and all data points outside a cone live within
the \emph{convex polar} to that cone \citep{HULL}.
The formal definition is as follows.

\begin{assumption}\label{ass:NC}
  There exist $(\beta_k)_{k=1}^r$ with $\|\beta_k\|=1$ and $\alpha_k \in \{\pm 1/k\}$
  and $\gnc > 0$ and $\eps\in (0,\gnc)$ so that almost surely over the
  draw of any data $((x_i,y_i))_{i=1}^n$, then for any particular $(x_i,y_i,\beta_k)$:
  \begin{itemize}
    \item
      either $\beta_k^\T x_i y_i \geq \gnc$
      and $\|(I-\beta_k\beta_k^\T)x_i\| \leq \gnc \sqrt{\eps/2}$
      (example $i$ lies in a narrow cone around $\beta_k$),
    \item
      or $\beta_k^\T x_iy_i \leq -\eps$
      (example $i$ lies in the polar of the cone around $\beta_k$).
      \qedhere
  \end{itemize}
\end{assumption}

It follows that \Cref{ass:NC} implies \Cref{ass:gl:M-D} with margin $\ggl \geq \gnc/k$,
but the condition is quite a bit stronger.  The corresponding GF result is as follows.

\begin{theorem}\label{fact:nc}
  Suppose the data distribution satisfies \Cref{ass:NC}
  for some $(r, \gnc, \eps)$, and let $\ell=\lexp$ be given.
  If the network width $m$ satisfies
  \[
    m \geq 2 \del{ \frac 2 \eps }^d \ln \frac {r}{\delta},
  \]
  then, with probability at least $1-3\delta$,
  the GF curve $(W_s)_{s\geq 0}$ for all large times $t$ satisfies
  \[
    \ngamma(W_t) \geq \frac {\gnc-\eps}{8r}
    \qquad\textup{and}\qquad
    \Pr\sbr{p(x,y;W_t) \leq 0}
    =
    \cO\del{
      \frac {r^2 \ln(n)^3}{n(\gnc-\eps)^2}
      + \frac {\ln \frac 1 \delta}{n}
    }.
  \]
\end{theorem}

Note that \Cref{fact:nc} only implies that GF selects a predictor with \emph{margins}
at least as good as the NC solution, and does not necessarily converge to the NC solution
(i.e., rotating all ReLUs to point in the directions $(\beta_k)_{k=1}^r$ specified
by \Cref{fact:nc}).
In fact, this may fail to be true, and \Cref{fact:gf:kkt} and \Cref{fig:kkt} already
gave one such construction; moreover, this is not necessarily bad, as GF may converge
to a solution with \emph{better} margins, and potentially better generalization.
Overall, the relationship of NC to the bias of 2-layer network training in practical
regimes remains open.

The proof of \Cref{fact:nc} proceeds by developing a potential functions that
asserts that either mass grows in the directions $(\beta_k)_{k=1}^r$, or their margin
is exceeded.  Within the proof, large width ensures that the mass in each good direction
is initially positive, and thereafter \Cref{ass:NC} is used to ensure that the fraction
of mass in these directions is increasing.  The proof of \Cref{fact:nc} and of
\Cref{fact:gl} invoke the same abstract potential function \namecref{fact:abstract:Phi},
and discussion is momentarily deferred until after the presentation of \Cref{fact:gl}.

One small point is worth explaining now, however.  It may have seemed unusual to use
$\|a_jv_j\|$ as a (squared!) norm in \Cref{fig:2xor}.
Of course, layers asymptotically balance,
thus asymptotically not only is there the Fenchel inequality
$2\|a_jv_j\| \leq a_j^2 + \|v_j\|^2 = \|w_j^2$, but also a reverse inequality
$2\|a_jv_j\| \gtrsim \|w_j\|^2$.  Despite this fact, the disagreement
between $2\|a_jv_j\|$ and $\|w_j\|^2$, namely the imbalance between $a_j^2$ and $\|v_j\|^2$,
can cause real problems, and one solution used within the proofs is to replace
$\|w_j\|^2$ with $\|a_jv_j\|$.

\subsection{Global margin maximization}\label{sec:gl}

The final \namecref{fact:gl} will be on stylized networks where the inner layer is
forced to not rotate.  Specifically, the networks are of the form
\[
  x\mapsto \sum_j a_j \sigma(b_j v_j^\T x),
\]
where $((a_j,b_j))_{j=1}^m$ are trained, but $v_j$ are fixed at initialization;
the new scalar parameter $b_j$ is effectively the norm of $v_j$
(though it is allowed to be negative).
As a further simplification, $a_j$ and $b_j$ are initialized to have the
same norm; this initial balancing is common in many implicit bias proofs, but is impractical
and constitutes a limitation to improve in future work.  While it is clearly
unpleasant that $(v_j)_{j=1}^m$ can not rotate, \Cref{fig:2xor:3} provides some hope
that this is approximated in networks of large width.

\begin{theorem}\label{fact:gl}
  Suppose the data distribution satisfies \Cref{ass:gl:M-D} for some $\ggl>0$
  with reference architecture $((\alpha_k,\beta_k))_{k=1}^r$.
  Consider the architecture
  $x\mapsto \sum_j a_j \sigma(b_j v_j^\T x_i)$
  where
  $((a_j(0),b_j(0)))_{j=1}^m$ are sampled uniformly from the two choices $\pm 1/m^{1/4}$,
  and $v_j(0)$ is sampled from the unit sphere
  (e.g., first $v'_j\sim \cN_d/\sqrt{d}$, then $v_j(0) := v'_j / \|v'_j\|$),
  and
  \[
    m \geq 2 \del{\frac {4}{\ggl} }^d \ln \frac r \delta.
  \]
  Then, with probability at least $1-3\delta$,
  for all large $t$, GF on $((a_j,b_j))_{j=1}^m$ with $\lexp$ satisfies
  \[
    \ngamma\del{(a(t),b(t))} \geq \frac {\ggl}{4}
    \qquad\textup{and}\qquad
    \Pr\sbr{p(x,y;(a(t),b(t))) \leq 0}
    =
    \cO\del{
      \frac {\ln(n)^3}{n\ggl^2}
      + \frac {\ln \frac 1 \delta}{n}
    }.
  \]
\end{theorem}

The main points of comparison for \Cref{fact:gl} are the global margin maximization proofs
of \citet{wei_reg} and \citet{chizat_bach_imp}.  The analysis by \citet{wei_reg}
is less similar, as it heavily relies upon the benefits to local search arising from
weight re-initialization, whereas the analysis here in some sense is based on the
technique in \citep{chizat_bach_imp}, but diverges sharply due to dropping the two key
assumptions therein.  Specifically, \citep{chizat_bach_imp} requires infinite width
and \emph{dual convergence}, meaning $q(t)$ converges, which is open
even for linear predictors in general settings.  The infinite width assumption is also
quite strenuous: it is used to ensure not just that weights cover the sphere at initialization
(a consequence of exponentially large width), but in fact that they cover the sphere
\emph{for all times $t$}.

The proof strategy of \Cref{fact:gl} (and \Cref{fact:nc}) is as follows.  The core of
the proof scheme in \citep{chizat_bach_imp} is to pick two weights $w_j$ and $w_k$,
where $w_k$ achieves better margins than $w_j$ in some sense, and consider
\[
  \frac {\dif}{\dif t } \frac {\|w_j\|^2}{\|w_k\|^2}
  =
  4\cQ \del{\frac {\|w_j\|^2}{\|w_k\|^2}}
  \sum_i q_i \sbr{ \tp_i(w_j) - \tp_i(w_k) };
\]
as a purely technical aside, it is extremely valuable that this ratio potential
automatically normalizes the margins, leading to the appearance of $\tp_i(w_j)$ not
$p_i(w_j)$, and a similar idea is used in the proofs here, albeit starting
from $\ln\|w_j\| - \ln\|w_k\|$, the idea of $\ln(\cdot)$ also appearing in
the proofs by \citet{kaifeng_jian_margin}.  Furthermore, this expression already
shows the role of dual convergence: if we can assume every $q_i$ converges, then
we need only pick nodes $w_k$ for which the margin surrogate
$\sum_i q_i(\infty) \tp_i\del{w_k(\infty)}$
is very large,
and the above time derivative becomes negative
if $w_j$ has bad margin, which implies mass accumulates in directions with good margin.
This is the heart of the proof scheme due to \citep{chizat_bach_imp}, and circumventing
or establishing dual convergence seems tricky.

The approach here is to replace $\|w_j\|$ and $\|w_k\|$ with other quantities which
can be handled without dual convergence.  First, $\|w_j\|$ is replaced with $\|W\|$, the
norm of all nodes, which can be controlled in an elementary way for arbitrary
$L$-homogeneous networks: as summarized in \Cref{fact:ngamma:ub}, as soon as $\|W\|$
becomes large, then $\sum_i q_i \tp_i(W)\approx \gamma(W_t)$, essentially by
properties of $\ln\sum\exp$.

Replacing $\ln\|w_k\|$ is much harder, since $q_i(t)$ may oscillate and thus the notion
of nodes with good margin seems to be time-varying.  If there is little rotation,
then nodes near the reference directions $(\beta_k)_{k=1}^r$ can be swapped with
$(\beta_k)_{k=1}^r$, and the expression $\sum_i q_i \tp_i(w_k)$ can be swapped with
$\ggl$.  A potential function that replaces $\ln\|w_k\|$ and allows this swapping
need only satisfy a few abstract but innocuous conditions, as summarized in
\Cref{fact:abstract:Phi}.  Unfortunately, verifying these conditions is rather painful,
and handling general settings (without explicitly disallowing rotation) seems to
still need quite a few more ideas.

\section{Concluding remarks and open problems}\label{sec:open}

This work provides settings where SGD and GF can select good features,
but many basic questions and refinements remain.

\Cref{fig:2xor} demonstrated low rotation with 2-sparse parity;
can this be proved, thereby establishing \Cref{fact:gl} without
forcing nodes to not rotate?

\Cref{fact:sgd} and \Cref{fact:gf:margins} achieve the same sample complexity for SGD
and GF, but via drastically different proofs, the GF proof being weirdly complicated;
is there a way to make the two more similar?

Looking to \Cref{table:2xor} for 2-sparse parity, the approaches here fail to achieve the
lowest width; is there some way to achieve this with SGD and GF, perhaps even
via margin analyses?

The approaches here are overly concerned with reaching a constant factor of the optimal
margins; is there some way to achieve slightly worse margins with the benefit of
reduced width and computation?  More generally, what is the Pareto frontier of
width, samples, and computation in \Cref{table:2xor}?

The margin analysis here for the logistic loss, namely
\Cref{fact:gf:margins}, requires a long warm start phase.  Does this reflect
practical regimes?  Specifically, does good margin maximization and feature learning
occur with the logistic loss in this early phase?

\subsubsection*{Acknowledgements}
The author thanks Peter Bartlett, Spencer Frei, Danny Son, and Nati Srebro for
discussions.  The author thanks the Simons Institute for hosting a short visit
during the 2022 summer cluster on Deep Learning Theory,
the audience at the corresponding workshop
for exciting and clarifying participation during a talk presenting this
work,
and the NSF for support under grant IIS-1750051.

\bibliography{bib}
\bibliographystyle{plainnat}

\appendix

\section{Technical preliminaries}
\label{sec:app:tech}

As follows are basic technical tools used throughout.

\subsection{Estimates of $\gntk$ and $\ggl$}

This section provides estimates of $\gntk$ and $\ggl$ in various settings.
The first estimate is of linear predictors.

\begin{proof}[Proof of \Cref{fact:gamma:linear}]
  The proof considers the three settings separately.
  \begin{enumerate}
    \item
      For any $i$, first note that
      \begin{align*}
        \bbE_{w \sim \cN_\theta}
        \ip{\theta(w)}{\hs_{w} p_i(w) }
        &=
        \bbE_{(a,v) \sim \cN_\theta}
        |a| \baru^\T x_i y_i \sigma'(v^\T x_i) \1[\|(a,v)\| \leq 2]
        \\
        &\geq
        \hgamma
        \bbE_{(a,v) \sim \cN_\theta}
        |a| \sigma'(v^\T x_i) \1[\|(a,v)\| \leq 2].
      \end{align*}
      To control the expectation, note that with probability at least $1/2$, then
      $1/4 \leq |a| \leq \sqrt{2}$, and thus by rotational invariance
      \begin{align*}
        \bbE_{(a,v) \sim \cN_\theta}
        |a| \sigma'(v^\T x_i) \1[\|(a,v)\| \leq 2]
        &\geq
        \frac 1 {8}
        \bbE_{(a,v) \sim \cN_\theta}
        \sigma'(v^\T x_i) \1[\|v\| \leq \sqrt{2}]
        \\
        &\geq
        \frac 1 {8}
        \bbE_{(a,v) \sim \cN_\theta}
        \sigma'(v_1) \1[\|v\| \leq \sqrt{2}]
        \\
        &\geq
        \frac 1 {32}.
      \end{align*}

    \item
      For convenience, fix any example $(x,y) \in ((x_i,y_i))_{i=1}^n$,
      and write $(a,v) = w$, whereby $w\sim \cN_w$ means $a\sim\cN_a$ and $v\sim\cN_v$.
With this out of the way, define orthonormal matrix
      $M \in \R^{d\times d}$ where the first column is $\baru$, the second column is
      $(I - \baru\baru^\T)x/\|(I-\baru\baru^\T)x\|$, and the remaining columns are arbitrary so long as
      $M$ is orthonormal, and note that $Mu = e_1$ and
      $Mx = e_1 \baru^\T x + e_2 r_2$ where $r_2 := \sqrt{\|x\|^2 - (\baru^\T x)^2}$.
      Then,
      using rotational invariance of the Gaussian,
      \begin{align*}
        \bbE_w \ip{\theta(w)}{\hs p_i(w)}
        &=
        y \bbE_{w = (a,v)} \sgn(\baru^\T v) \sigma( v^\T x )\1[\|w\|\leq 2]
        \\
        &=
        y \bbE_{\|(a,Mv)\|\leq 2} \alpha(Mv) \sigma(v^\T M^Tx)
        \\
        &=
        \bbE_{\|(a,v)\|\leq 2} y \sgn(v_1)  \sigma(v_1 \baru^\T x y^2 + v_2 r_2)
        \\
        &=
        \bbE_{\|(a,v)\|\leq 2} y \sgn(v_1)  \sigma(y \sgn(v_1) |v_1| \baru^\T x y + v_2 r_2)
        \\
        &= \bbE_{\substack{\|(a,v)\|\leq 2\\y\sgn(v_1)=1\\v_2 \geq 0}}
        \Big[
          \sigma(|v_1| \baru^\T x y + v_2 r_2)
         -  \sigma(- |v_1| \baru^\T x y + v_2 r_2)
         \\
         &\qquad\qquad
         + \sigma(|v_1| \baru^\T x y - v_2 r_2)
         -  \sigma(- |v_1| \baru^\T x y - v_2 r_2)
       \Big].
      \end{align*}
      Considering cases, the first ReLU argument is always positive, exactly one of the second and third
      is positive, and the fourth is negative, whereby
      \begin{align*}
        y \bbE_{\|(a,v)\|\leq 2} \alpha(v) \sigma(v^\T x)
        &= \bbE_{\substack{\|(a,v)\|\leq 2\\y\sgn(v_1)=1\\v_2 \geq 0}}
        \sbr{ |v_1| \baru^\T xy + v_2 r_2 + |v_1| \baru^\T xy - v_2r_2}
        \\
        &=  2 \bbE_{\substack{\|(a,v)\|\leq 2\\y\sgn(v_1)=1}}
         |v_1| \baru^\T xy
        \\
        &\geq
        2\hgamma
        \bbE_{\substack{\|v\|\leq 1\\y\sgn(v_1)=1}}
         |v_1|
         \\
        &=
        \hgamma
        \Pr[\|(a,v)\| \leq 2]
        \bbE\del{
          |v_1| 
        \ \big|  \ \|(a,v)\|\leq 2},
      \end{align*}
      where $\Pr[\|(a,v)\|\leq 2]\geq 1/4$ since (for example) the $\chi^2$ random variables
      corresponding to $|a|^2$ and $\|v\|^2$ have median less than one,
      and the expectation term is at least $1/(4\sqrt{d})$ by standard Gaussian computations
      \citep[Theorem 2.8]{bhk_data_science}.
    \item
      For any pair $(x_i,y_i)$,
      \begin{align*}
        2 y_i \sum_{j=1}^2 \alpha_j \sigma(\beta_j^\T x_i)
        &=
        y_i \sigma(\baru^\T x_i) - y_i \sigma(-\baru^\T x_i)
        \\
        &= \1[y_i = 1] \sigma(y_i \baru^\T x_i) + \1[y_i = -1]\sigma(y_i \baru^\T x_i)
        \\
        &= y_i \baru^\T x_i
        \\
        &\geq \hgamma.
      \end{align*}
    \end{enumerate}
\end{proof}

Next, the construction for $2$-sparse parity.
As is natural in maximum margin settings, but in contrast with most
studies of sparse parity, only the support of the distribution matters (and the labeling),
but not the marginal distribution of the inputs.

\begin{proof}[Proof of \Cref{fact:margins:parity}]
  This proof shares ideas with \citep{wei_reg,ziwei_ntk}, though with some
  adjustments to exactly fit the standard 2-sparse parity setting, and to shorten the proofs.

  Without loss of generality, due to the symmetry of the data distribution about the origin,
  suppose $a=1$ and $b=2$, meaning for any $x\in H_d$, the correct label is $dx_1x_2$,
  the product of the first two coordinates.
  Both proofs will use the global margin construction (the parameters for $\ggl$), given
  as follows: 
  $p(x,y;(\alpha,\beta)) = y \sum_{j=1}^4 \alpha_j\sigma(\beta_j^\T x)$,
  where $\alpha = (1/4, -1/4, -1/4, 1/4)$ and
  \begin{align*}
    \beta_1 &:= \del{ \frac 1 {\sqrt{2}}, \frac {1}{\sqrt 2}, 0,\ldots, 0} \in \R^d,
    \\
    \beta_2 &:= \del{ \frac 1 {\sqrt{2}}, \frac {-1}{\sqrt 2}, 0,\ldots, 0} \in \R^d,
    \\
    \beta_3 &:= \del{ \frac {-1} {\sqrt{2}}, \frac {1}{\sqrt 2}, 0,\ldots, 0} \in \R^d,
    \\
    \beta_4 &:= \del{ \frac {-1} {\sqrt{2}}, \frac {-1}{\sqrt 2}, 0,\ldots, 0} \in \R^d.
  \end{align*}
  Note moreover that for any $x\in H_d$, then $\beta_j^\T x > 0$ for exactly one $j$,
  which will be used for both $\gntk$ and $\ggl$.
  The proof now splits into the two different settings, and will heavily use symmetry
  within $H_d$ and also within $(\alpha,\beta)$.
  \begin{enumerate}
    \item
      Consider the transport mapping
      \[
        \theta\del{(a,v)}
        = \del{0, \frac {\sgn(a)} 2 \sum_{j=1}^4 \beta_j \1[\beta_j^\T v \geq 0]};
      \]
      note that this satisfies the condition $\|\theta(w)\|\leq 1$ thanks to the factor
      $1/2$, since each $\beta_j$ gets a hemisphere, and $(\beta_1,\beta_4)$ together
      partition the sphere once, and $(\beta_2,\beta_3)$ similarly together partition
      the sphere once.

      Now let any $x$ be given, which as above has label $y = x_1x_2$.  By rotational
      symmetry of the data and also the transport mapping, suppose suppose $\beta_1$
      is the unique choice with $\beta_1^\T x > 0$, which implies $y = 1$,
      and also $\beta_2^\T x = 0 = \beta_3^\T x = 0$, however $\beta_4^\T x = -\beta_4^\T x$.
      Using these observations, and also rotational invariance of the Gaussian,
      \begin{align*}
        \hspace{1in}&\hspace{-1in}
        \bbE_{a,v} \ip{\theta(a,v)}{\cs p(x,y;w)}
        \\
        &=
        \bbE_{a,v} \frac {|a|}{2} \sum_{j=1}^4 \beta_j^\T x
        \1[\beta_j^\T v \geq 0]\cdot\1[v^\T x \geq 0]
        \\
        &=
        \beta_1^\T x \del{\bbE_a \frac {|a|}{2}}
        \cdot
        \del{
          \bbE_{v}\1[\beta_1^\T v \geq 0]\cdot\1[v^\T x \geq 0]
          -
          \bbE_{v}\1[-\beta_1^\T v \geq 0]\cdot\1[v^\T x \geq 0]
        }.
      \end{align*}
      Now consider $\bbE_v \1[\beta_1^\T v\geq 0]\cdot\1[v^\T x\geq 0]$.
      A standard Gaussian computation is to introduce a rotation matrix $M$
      whose first column is $\beta_1$, whose second column is
      $(I - \beta_1\beta_1^\T)x / \|(I-\beta_1\beta_1^\T)x\|$, and the rest
      are orthogonal, which by rotational invariance and the calculation
      $\beta_1^\T x = \sqrt{2/d}$ gives
      \begin{align*}
        \bbE_v \1[\beta_1^\T v\geq 0]\cdot\1[v^\T x\geq 0]
        &= \bbE_v \1[\beta_1^\T M v\geq 0]\cdot\1[v^\T M x\geq 0]
        \\
        &= \bbE_v \1[v_1 \geq 0]\cdot\1[v_1 \beta_1^\T x + v_2 \sqrt{1-(\beta_1^\T x)^2}
        \\
        &= \bbE_v \1[v_1 \geq 0]\cdot\1[v_1 + v_2 \sqrt{d/2 - 1} \geq 0].
      \end{align*}
      Performing a similar calculation for the other term (arising from $\beta_4^\T x$)
      and plugging all of this back in,
      \begin{align*}
        \hspace{0.5in}&\hspace{-0.5in}
        \bbE_{a,v} \ip{\theta(a,v)}{\cs p(x,y;w)}
        \\
        &=
        \sqrt{\frac 2 d}
        \del{\bbE_a \frac {|a|}{2}}
        \cdot
        \bbE_{v}
        \1[v_1 \geq 0]
        \del{
          \1[v_1 + v_2 \sqrt{d/2-1} \geq 0]
          -
          \1[-v_1 + v_2 \sqrt{d/2-1} \geq 0]
        }.
      \end{align*}
      To finish, a few observations suffice.
      Whenever $v_1\geq 0$
      (which is enforced by the common first term),
      then $-v_1 + v_2\tau \leq v_1 + v_2\sqrt{d/2-1}$, so the first indicator is $1$ whenever
      the second indicator is $1$, thus their difference is nonnegative, and to lower
      bound the overall quantity, it suffices to asses the probability that
      $v_1 + v_2 \sqrt{d/2-1} \geq 0$ whereas $-v_1 + v_2 \sqrt{d/2-1} \leq 0$.
      To lower bound this event, it suffices to lower bound
      \[
        \Pr[v_1 \geq 0 \ \land \ v_2 \geq 0 \ \land \ v_1 \geq v_2\sqrt{d/2-1}]
        \geq
        \Pr[v_1 \geq \sqrt{1/2}] \cdot \Pr[ 0 \leq v_2 \leq \sqrt{1/d}.
      \]
      The first term is at least $1/5$, and the second can be calculated via brute force:
      \[
        \Pr[v_2 \geq 1/\sqrt{d}]
        =
        \frac 1 {\sqrt{2\pi}} \int_0^{1/\sqrt{d}} \exp(-x^2)\dif x
        \geq
        \frac 1 {\sqrt{2\pi}} \int_0^{1/\sqrt{d}} \exp(-1/d)\dif x
        \geq
        \frac 1 {\sqrt{2\pi}} \del{\frac 1 {\sqrt d }} \frac 1 e,
\]
      which completes the proof after similarly using $\bbE_a |a| \geq 1$,
      and simplifying the various constants.

    \item
      Let any $x\in H_d$ be given,
      and as above note that $\beta_j^\T x>0$ for exactly one $j$.
      By symmetry, suppose it is $\beta_1$, whereby $y = x_1x_2=1$, and
      \[
        \ggl
        \geq
        p(x,y;(\alpha,\beta)) = y \sum_j \alpha_j \sigma(\beta_j^\T x)
        = |\alpha_1|\cdot \beta_1^\T x
        = \frac 1 4 \del{\frac {2}{\sqrt{2d}}}
        = \frac {1}{\sqrt{8d}}.
      \]

  \end{enumerate}
\end{proof}

Lastly, an estimate of $\gntk$ in a simplified version of \Cref{ass:NC},
which is used in the proof of 
\Cref{fact:gf:kkt}.

\begin{lemma}\label{fact:gamma:cones}
  Suppose \Cref{ass:gl:M-F} holds for data
  $((x_i,y_i))_{i=1}^n$ with reference solution $((\alpha_k,\beta_k))_{k=1}^r$
  and margin $\ggl > 0$, and additionally $\alpha_k > 0$ and $y_i = +1$ and $\|x_i\|=1$.
  Then \Cref{ass:M-F} holds with margin $\gntk \geq \ggl/(8\sqrt d)$.
\end{lemma}
\begin{proof}
  Define $\theta(a,v) := \del{\sum_{k=1}^r \alpha_k \sigma'(\beta_k^\T v), 0 }\1[\|(a,v)\|\leq 2]$.
  Fix any $x \in (x_i)_{i=1}^n$,
  and for each $k$
  define orthonormal matrix $M_k$ with first column $\beta_k$ and
  second column $(I - \beta_k\beta_k^\T)x/\|(I-\beta_k\beta_k^\T)x\|$,
  whereby $M_k\beta_k = e_1$ and
  $M_kx = e_1 \beta_k^\T x + e_2 r_j$ where $r_j := \sqrt{1 - (\beta_k^\T x)^2}$.
  Then,
  using rotational invariance of the Gaussian and Jensen's inequality applied to the ReLU,
  \begin{align*}
    y \bbE_{w\sim\cN_\theta} \ip{\theta(w)}{ \cs\sigma p_i(w) }
    &=
    \bbE_{\|w\|\leq 2} \sum_{k=1}^r\alpha_k \sigma'(\beta_k^\T v) \sigma(v^\T x)
    \\
    &=
    \sum_k \alpha_k
    \bbE_{\|M_j^\T v\|\leq 2} \sigma'(v^\T M_ku_k) \sigma(v^\T M_k x)
    \\
    &=
    \sum_k \alpha_k
    \bbE_{\substack{\|v\|\leq 1\\v_1\geq 0}} \sigma(v_1 \beta_k^\T x + v_2 r_2)
    \\
    &\geq
    \sum_k \alpha_k
    \sigma\del[2]{\bbE_{\substack{\|v\|\leq 1\\v_1\geq 0}} v_1 u_k^\T x + v_2 r_2 }
    \\
    &=
    \sum_k \alpha_k
    \beta_k^\T x 
    \sigma\del[2]{\bbE_{\substack{\|v\|\leq 1\\v_1\geq 0}} v_1}
    \\
    &\geq
    \ggl
    \bbE_{\substack{\|v\|\leq 1\\v_1\geq 0}} v_1,
  \end{align*}
  which is bounded below by $\ggl / (8\sqrt d)$ via similar arguments 
  to those in the proof of \Cref{fact:gamma:linear}.
\end{proof}

\subsection{Gaussian concentration}
\label{sec:gaussians}

The first concentration inequalities are purely about the initialization.

\begin{lemma}\label{fact:gaussians:init}
  Suppose $a \sim \cN_m / \sqrt{m}$ and $V \sim \cN_{m\times d} / \sqrt{d}$.
  \begin{enumerate}
    \item
      With probability at least $1-\delta$, then $\|a\|\leq 1 + \sqrt{2\ln(1/\delta)/m}$;
      similarly, with probability at least $1-\delta$,
      then $\|V\| \leq \sqrt{m} + \sqrt{2\ln(1/\delta)/d}$.
    \item
      Let examples $(x_1,\ldots,x_n)$ be given with $\|x_i\|\leq 1$.
      With probability at least $1-4\delta$,
      \[
        \max_i \envert{ \sum_j a_j \sigma(v_j^\T x_i) }
        \leq
        4\ln(n/\delta).
      \]
  \end{enumerate}
\end{lemma}
\begin{proof}
  \begin{enumerate}
    \item
      Rewrite $\tilde a := a \sqrt{m}$, so that $\tilde a \sim \cN_m$.
      Since $\tilde a \mapsto \|\tilde a\| / \sqrt{m} = \|a\|$
      is $(1/\sqrt{m})$-Lipschitz, then by Gaussian concentration,
      \citep[Theorem 2.26]{wainwright_hds},
      \begin{align*}
        \|a\| &= \|\tilde a\|/\sqrt{m}
        \\
              &\leq \bbE \|\tilde a\|/\sqrt{m} + \sqrt{2\ln(1/\delta)/m}
              \\
              &\leq \sqrt{\bbE \|\tilde a\|^2}/\sqrt{m} + \sqrt{2\ln(1/\delta)/m}
              \\
              & = 1 + \sqrt{2\ln(1/\delta)/m}.
      \end{align*}
      Similarly for $V$, defining $\tilde V := V\sqrt{d}$ whereby $\tilde V \sim \cN_{m\times d}$,
      Gaussian concentration grants
      \[
        \|V\| = \|\tilde V\|/\sqrt{d} \leq \sqrt{m} + \sqrt{2 \ln(1/\delta)/d}.
      \]

    \item
    Fix any example $x_i$, and constants $\eps_1>0$ and $\eps_2>0$ to be optimized at the
    end of the proof, and define $d_i := d/\|x_i\|^2$ for convenience.
    By rotational invariance of Gaussians and since $x_i$ is fixed,
    then $\sigma(Vx_i)$ is equivalent in distribution to
    $\|x_i\| \sigma(g)/\sqrt{d} = \sigma(g)/\sqrt{d_i}$ where $g\sim \cN_m$.
    Meanwhile, $g \mapsto \|\sigma(g)\| / \sqrt{d_i}$
    is $(1/\sqrt{d_i})$-Lipschitz with $\bbE \|\sigma(g)\| \leq \sqrt{m}$,
    and so, by Gaussian concentration
    \citep[Theorem 2.26]{wainwright_hds},
    \[
      \Pr[ \|\sigma(Vx_i)\| \geq \eps_1 + \sqrt{m} ]
      = 
      \Pr[ \| \sigma(g) \|/\sqrt{d_i}  \geq \eps_1 + \sqrt{m} ]
      \leq
      \exp\del{\frac {-d_i\eps_1^2}{2} }.
    \]
    Next consider the original expression $a^\T \sigma(Vx_i)$.
    To simplify handling of the $1/m$ variance of the coordinates of $a$,
    define another Gaussian $h := a \sqrt{m}$,
    and a new constant $c_i := m d_i$ for convenience,
    whereby $a^\T \sigma(V_i)$ is equivalent in distribution to
    equivalent in distribution to $h^\T \sigma(g)/\sqrt{c_i}$ since $a$ 
    and $V$ are independent (and thus $h$ and $V$ are independent).
    Conditioned on $g$, since $\bbE h = 0$, then $\bbE[ h^\T \sigma(g) | g] = 0$.
    As such, applying Gaussian concentration to this conditioned random variable,
    since $h \mapsto h^\T \sigma(g)/\sqrt{c_i}$ is $(\|\sigma(g)\|/\sqrt{c_i})$-Lipschitz,
    then
    \[
      \Pr[ h^\T \sigma(g)/\sqrt{c_i} \geq \eps_2\  \big| \  g ] \leq
      \exp\del{\frac{-c_i \eps_2^2}{2\|\sigma(g)\|^2}}.
    \]
    Returning to the original expression, it can now be controlled via the two preceding
    bounds, conditioning, and the tower property of conditional expectation:
    \begin{align*}
      \hspace{1em}&\hspace{-1em}\Pr[ h^\T \sigma(g)/\sqrt{c_i} \geq \eps_2 ]
      \\
      &\leq
      \Pr\sbr{ h^\T \sigma(g)/\sqrt{c_i} \geq \eps_2 \ \big | \  \|\sigma(g)\|/\sqrt{d_i} \leq \eps_1 + \sqrt{m}}
      \cdot
      \Pr\sbr{ \|\sigma(g)\|/\sqrt{d_i} \leq \eps_1 + \sqrt{m}}
      \\
      &\quad
      +
      \Pr\sbr{ h^\T \sigma(g)/\sqrt{c_i} \geq \eps_2 \ \big| \   \|\sigma(g)\|/\sqrt{d_i} > \eps_1 + \sqrt{m}}
      \cdot
      \Pr\sbr{ \|\sigma(g)\|/\sqrt{d_i} > \eps_1+\sqrt{m}}
      \\
      &=
      \bbE\sbr{\Pr[ h^\T \sigma(g)/\sqrt{c_i} \geq \eps_2 \ | \  g]  \ \Big | \  \|\sigma(g)\|/\sqrt{d_i} \leq \eps_1 + \sqrt{m}}
      \Pr[ \|\sigma(g)\|/\sqrt{d_i} \leq \eps_1 + \sqrt{m}]
      \\
      &\quad
      +
      \Pr[ h^\T \sigma(g)/\sqrt{c_i} \geq \eps_2  \ | \  \|\sigma(g)\|/\sqrt{d_i} > \eps_1 + \sqrt{m}]
      \Pr[ \|\sigma(g)\|/\sqrt{d_i} > \eps_1 + \sqrt{m}]
      \\
      &\leq
      \bbE\sbr{
      \exp\del{\frac{-c_i \eps_2^2}{2\|\sigma(g)\|^2}}
      \ \big| \  \|\sigma(g)\|/\sqrt{d_i}\leq \eps_1 + \sqrt{m}}
      +
      \exp\del{-d_i\eps_1^2/2}
      \\
      &\leq
      \exp\del{\frac{-c_i \eps_2^2}{4d_i\eps_1^2 + 4d_i m}}
      +
      \exp\del{-d_i\eps_1^2/2}.
    \end{align*}
    As such,
    choosing $\eps_2 := 4\ln(n/\delta)\sqrt{md_i/c_i} = 4\ln(n/\delta)$
    and $\eps_1 := \sqrt{2\ln(n/\delta)/d_i}$
    gives
    \[
      \Pr[ a^\T \sigma(Vx_i) \geq \eps_2 ]
      =
      \Pr[ h^\T \sigma(g)/\sqrt{c_i} \geq \eps_2 ]
      \leq
      \frac \delta n
      +
      \frac \delta n,
    \]
    which is a sub-exponential concentration bound.
    Union bounding over the reverse inequality and over all $n$ examples and using
    $\max_i\|x_i\|\leq 1$ gives the final bound.
  \end{enumerate}
\end{proof}

Next comes a key tool in all the proofs using $\gntk$:
guarantees that the infinite-width margin
assumptions imply the existence of good finite-width networks.
These bounds are stated for \Cref{ass:M-F}, however they will also be applied
with \Cref{ass:M-D}, since \Cref{ass:M-D} implies that \Cref{ass:M-F} holds almost surely
over any finite sample.

\begin{lemma}\label{fact:gaussians:margin}
  Let examples $((x_i,y_i))_{i=1}^n$ be given,
  and suppose \Cref{ass:M-F} holds, with corresponding $\theta : \R^{d+1} \to \R^{d+1}$
  and $\gntk>0$ given.
  \begin{enumerate}
    \item
      With probability at least $1-\delta$
      over the draw of $(w_j)_{j=1}^m$,
      defining $\btheta_j := \theta(w_j) / \sqrt{m}$,
      then
      \[
        \min_i \sum_j \ip{\btheta_j}{\hs p_i(w_j)}
        \geq
        \gntk\sqrt{m} - \sqrt{32\ln(n/\delta)}
        .
      \]

    \item
      With probability at least $1-7\delta$ over the draw of $W$ with rows $(w_j)_{j=1}^m$
      with $m\geq 2\ln(1/\delta)$, defining rows $\btheta_j := \theta(w_j)/\sqrt{m}$
      of $\btheta \in \R^{m\times (d+1)}$,
      then for any $W'$ and any $R \geq \|W-W'\|$ and any $r_\theta\geq 0$ and $r_w\geq 0$,
      \begin{align*}
        \ip{r_\theta \btheta + r_w W}{\hs p_i(W')} - r_w p_i(W')
        &\geq
        \gntk r_\theta \sqrt{m}
        - r_\theta \sbr{\sqrt{32 \ln(n/\delta)} + 8R + 4}
        \\
        &\quad
        - r_w \sbr{
          4 \ln(n/\delta)
          + 2R + 2R\sqrt{m} + 4\sqrt{m}
        },
      \end{align*}
      and moreover, writing $W = (a,V)$, then $\|a\|\leq 2$ and $\|V\|\leq 2\sqrt{m}$.
      For the particular choice $r_\theta := R/8$ and $r_w = 1$,
      if $R\geq 8$ and $m \geq (64 \ln(n/\delta)/\gntk)^2$,
then
      \[
        \ip{r_\theta \btheta + W}{\hs p_i(W')} - p_i(W')
        \geq
        \frac{\gntk r_\theta \sqrt{m}}{2}
        - 160 r_\theta^2.
      \]
  \end{enumerate}
\end{lemma}
\begin{proof}
  \begin{enumerate}
    \item
      Fix any example $(x_i,y_i)$, and define
      \[
        \mu := \bbE_w \ip{\theta(w)}{\hs p_i(w)},
      \]
      where $\mu \geq \gntk$ by assumption.
      By the various conditions on $\theta$,
      it holds for any $(a,v):= w\in\R^{d+1}$
      and corresponding $(\bara, \barv) := \theta(w)\in\R^{d+1}$ 
      that
\begin{align*}
        \envert{ \ip{\theta(w)}{\hs p_i(w)} }
        &\leq
        \envert{ \bara \sigma(v^\T x_i) }
        +
        \envert{ \ip{\barv}{a x_i \sigma'(v^\T x_i)} }
        \\
        &\leq
        |\bara| \cdot \1[\|v\|\leq 2] \cdot \|v\|\cdot \|x_i\|
        +
        \|\barv\|\cdot|a|\cdot \1[|a| \leq 2] \cdot \|x_i\|
        \\
        &\leq 4.
      \end{align*}
      and therefore, by Hoeffding's inequality, with probability at least $1 - \delta/n$
      over the draw of $m$ iid copies of this random variable,
      \[
        \sum_j \ip{\theta(w_j)}{\hs p_i(w_j)}
        \geq m \mu - \sqrt{32m\ln(n/\delta)}
        \geq m \gntk - \sqrt{32m\ln(n/\delta)},
      \]
      which gives the desired bound after dividing by $\sqrt{m}$,
      recalling $\btheta_j := \theta(w_j)/\sqrt{m}$,
      and union bounding over all $n$ examples.

    \item
      First, suppose with probability at least $1-7\delta$ that the consequences
      of \Cref{fact:gaussians:init}
      and the preceding part of the current lemma hold, whereby simultaneously
      $\|a\| \leq 2$, and $\|V\| \leq 2\sqrt{m}$,
      and
      \[
        \min_i p_i(W) \geq - 4\ln(n/\delta),
        \qquad
        \textup{and}
        \qquad
        \min_i \sum_j \ip{\btheta_j}{\hs p_i(w_j)} \geq \gntk \sqrt{m} - \sqrt{32\ln(n/\delta)}.
      \]
      The remainder of the proof proceeds by separately lower bounding the two
      right hand terms in
      \begin{align*}
        \ip{r_\theta \btheta + r_w W}{\hs p_i(W')} - r_w p_i(W')
        &=
        r_\theta \sbr{ \ip{\btheta}{\hs p_i(W)} + \ip{\btheta}{\hs p_i(W') - \hs p_i(W) } }
        \\
        &\quad
        + r_w \sbr{
          \ip{W}{\hs p_i(W')} - r_w p_i(W')
        }
        .
      \end{align*}
      For the first term, writing $(\bara,\barV) = \btheta$ and noting $\|\bara\|\leq 2$
      and $\|\barV\|\leq 2$, then for any $W' = (a', V')$,
      \begin{align*}
        \envert{ \ip{\btheta}{\hs p_i(W') - \hs p_i(W) } }
        &\leq
        \envert{ \sum_j \bara_j \del{\sigma(x_i^\T v'_j) - \sigma(v_j^\T x_i)} }
        \\
        &\quad
        +
        \envert{
          \sum_j x_i^\T \barv_j \del{a'_j \sigma'(x^\T v'_j) - a_j\sigma'(x^\T v_j) }
        }
        \\
        &\leq
        \sqrt{ \sum_j \bara_j^2 }
        \sqrt{\sum_j \del{\sigma(x_i^\T v'_j) - \sigma(v_j^\T x_i)}^2 }
        \\
        &\quad+
        \sum_j | x_i^\T \barv_j|\cdot \envert{ a'_j \sigma'(x^\T v'_j) - a_j\sigma'(x^\T v_j') }
        \\
        &
        \quad
        +
        \sum_j | x_i^\T \barv_j|\cdot \envert{ a_j \sigma'(x^\T v_j') - a_j\sigma'(x^\T v_j') }
        \\
        &\leq
        \|\bara\|\cdot \|V' - V\|
        +
        \|a' - a\|\cdot \|\barV\|
        +
        \|a\|\cdot \|\barV\|
        \\
        &\leq
        4 R + 4.
      \end{align*}
      For the second term,
      \begin{align*}
        \envert{\ip{W}{\hs p_i(W')} - p_i(W')}
        &=
        \envert{\ip{a}{\hs_a p_i(W')} + \ip{V}{\hs_Vp_i(W')} - \ip{V'}{\hs_Vp_i(W')} }
        \\
        &\leq
        \envert{\sum_j a_j \sigma(x_i^\T v_j')}
        +
        \envert{\sum_j a_j' \ip{v_j - v_j'}{x_i} \sigma'(x_i^\T v_j)}
        \\
        &\leq
        \envert{p_i(w) + y_i\sum_j a_j \del{\sigma(x_i^\T v_j') - \sigma(x_i^\T v_j)}}
        +
        \sum_j \envert{ a_j'} \cdot \|v_j - v_j'\|
        \\
        &\leq
        4 \ln(n\delta)
        + 
        \|a\|\cdot \|V- V'\|
        +
        \|a' - a + a\|\cdot \|V-V'\|
        \\
        &\leq
        4 \ln(n\delta)
        + 4 R + R^2.
      \end{align*}
      Multiplying through by $r_\theta$ and $r$ and combining these inequalities gives,
      for every $i$,
      \begin{align*}
        \ip{r_\theta \btheta + r_w W}{\hs p_i(W')} - r_w p_i(W')
        &\geq
        \gntk r_\theta \sqrt{m}
        - r_\theta \sbr{\sqrt{32 \ln(n/\delta)} + 4R + 4}
        \\
        &\quad
        - r_w \sbr{
          4 \ln(n/\delta)
          + 4R + R^2
        },
      \end{align*}
      which establishes the first inequality.
      For the particular choice $r_\theta := R/8$ with $R\geq 8$ and $r_w = 1$,
      and using $m \geq (64 \ln(n/\delta)/\gntk)^2$,
      the preceding bound simplifies to
      \begin{align*}
        \ip{r_\theta \btheta + r_w W}{\hs p_i(W')} - r_w p_i(W')
        &\geq
        \gntk r_\theta \sqrt{m}
        - r_\theta \sbr{\frac {\gntk \sqrt{m}}{8} + 32r_\theta + 32r_\theta}
        \\
        &\quad
        - \sbr{
          \frac {\gntk \sqrt{m}}{16}
          + 32r_\theta + 64r_\theta^2 
        }
        \\
        &\geq
        \frac{\gntk r_\theta \sqrt{m}}{2}
        - 160 r_\theta^2.
      \end{align*}
  \end{enumerate}
\end{proof}

\subsection{Basic properties of $L$-homogeneous predictors}

This subsection collects a few properties of general $L$-homogeneous predictors
in a setup more general than the rest of the work, and used in all large margin
calculations.  Specifically, suppose general parameters $u_t$ with 
some unspecified initial condition $u_0$, and thereafter given by the differential
equation
\begin{align}
  \dot u_t
  &= - \cs_u \hcR(p(u_t)),
  &
  p(u)
  &:= (p_1(u), \ldots,p_n(u))\in\R^n,
  \label{eq:homo}\\
       &&
  p_i(u) &:= y_i F(x_i;u),
  \notag\\
         &&
  F(x_i;cu) &= c^L F(x_i;u)&\forall c \geq 0.
  \notag
\end{align}

The first property is that norms increase once there is a positive margin.

\begin{lemma}[{Restatement of \citep[Lemma B.1]{kaifeng_jian_margin}}]\label{fact:gf:norm_growth}
  Suppose the setting of \cref{eq:homo} and also $\ell\in\{\lexp,\llog\}$.
  If $\cR(u_\tau) < \ell(0)/n$, then, for every $t\geq \tau$,
  \[
    \frac {\dif}{\dif t} \|u_t\| > 0
    \qquad\textup{and}\qquad
    \ip{u_t}{\dot u_t} > 0,
  \]
  and moreover $\lim_t \|u_t\| = \infty$.
\end{lemma}
\begin{proof}
Since $\hcR$ is nonincreasing during gradient flow,
  it suffices to consider any $u_s$
  with $\hcR(u_s) < \ell(0)/n$.  To apply \citep[Lemma B.1]{kaifeng_jian_margin},
  first note that both the exponential and logistic losses can be handled,
  e.g., via the discussion of the assumptions at the beginning of \citep[Appendix A.1]{kaifeng_jian_margin}.
  Next, the statement of that lemma is
  \[
    \frac \dif {\dif s} \ln \|u_s\| > 0,
  \]
  but note that $\|u_s\|>0$ (otherwise $\hcR(u_s)< \ell(0)/n$ is impossible), and also that
  \[
    \frac \dif {\dif s} \|u_s\|
    = \frac{\ip{u_s}{\dot u_s}}{\|u_s\|},
    \quad\textup{and}\quad
    \frac \dif {\dif s} \ln \|u_s\| = \frac {\ip{u_s}{\dot u_s}}{\|u_s\|^2},
  \]
  which together with $(\dif / \dif s)\ln\|u_s\| > 0$ from \citep[Lemma B.1]{kaifeng_jian_margin} imply the main part of the
  statement; all that remains to show is $\|u_s\|\to\infty$, but this is given by \citep[Lemma B.6]{kaifeng_jian_margin}.
\end{proof}

Next, even without the assumption $\hcR(u_s) < \ell(0)/n$ (which at a minimum requires
a two-phase proof, and certain other annoyances), note that once $\|u_s\|$ is large,
then the gradient can be related to margins, even if they are negative, which
will be useful in circumventing the need for dual convergence and other assumptions
present in prior work (e.g., as in \citep{chizat_bach_imp}).

\begin{lemma}[name={See also \citep[Proof of Lemma C.5]{dir_align}}]\label{fact:ngamma:ub}
  Suppose the setting of \cref{eq:homo} and also $\ell = \lexp$.
  Then, for any $u$ and any $((x_i,y_i))_{i=1}^n$
  (and corresponding $\hcR$),
  \[
    \frac{\ip{u}{-n\hs_u\hcR(u)}}{L \|u\|^L}
    \leq
    \cQ \sbr{ \ngamma(u) + \frac{\ln n}{\|u\|^L}}
    \leq
    \cQ \sbr{ \gamma(u) + \frac{\ln n}{\|u\|^L} }.
  \]
\end{lemma}
\begin{proof}
  Define $\pi(p) = -\ln\sum\exp(-p) = \tgamma(u)$, whereby $q = \nabla_p \pi(p)$.
  Since $\pi$ is concave in $p$,
  \begin{align*}
    \ip{u}{-n\hs_u \hcR(u)}
    &=
    \sum_i |\ell'_i| \ip{u}{\cs p_i}
    =
    L \cQ\sum_i q_i p_i
    =
    L \cQ \ip{\nabla_p \pi(p)}{p}
    \\
    &=
    L \cQ \ip{\nabla_p \pi(p)}{p-0}
    \leq
    L\cQ \sbr{\pi(p) - \pi(0)}
    =
    L\cQ \sbr{\tgamma + \ln n}.
  \end{align*}
  Moreover, by standard properties of $\pi$,
  letting $k$ be the index of any example with $p_k(u) = \min_i p_i(u)$,
  \[
    \tgamma = -\ln \sum \exp(-p)
    \leq - \ln \exp(-p_k)
    = p_k
    = \gamma(u)\|u\|^L.
  \]
  Combining these inequalities and dividing by $L\|u\|^L$ gives the desired bounds.
\end{proof}

Lastly, a key abstract potential function lemma: this potential function is a proxy
for mass accumulating on certain weights with good margin, and once it satisfies a few 
conditions, large margins are implied directly.  This is the second component needed to
remove dual convergence from \citep{chizat_bach_imp}.

\begin{lemma}\label{fact:abstract:Phi}
  Suppose the setup of \cref{eq:homo} with $L=2$, and additionally that 
  there exists a constant $\hgamma > 0$, a time $\tau$,
  and a potential function $\Phi(u)$ so that $\Phi(u_\tau) > -\infty$,
  and for all $t\geq \tau$,
  \begin{align*}
    \Phi(u)
    &\leq \frac 1 L \ln \|u\|,
    \\
    \frac \dif {\dif t} \Phi(u)
    &\geq \cQ(u) \hgamma.
  \end{align*}
  Then it follows that $\hcR(u)\to\infty$,
  and $\|u\|\to\infty$, and $\int_\tau^t \cQ(u_s)\dif s = \infty$,
  and $\liminf_t \gamma(u_t)  \geq \hgamma$.
\end{lemma}
\begin{proof}
  First it is shown that if $\inf_s \hcR(u_s) > 0$
  (which is well-defined since since $\hcR$ is nonincreasing and bounded
  below by $0$), then $\int_\tau^\infty \cQ_s\dif s = \infty$ and $\|u\|\to \infty$.
  Since $\hcR(u_s) = \frac 1 n \cQ_s$,
  this implies $\inf_s \cQ_s > 0$,
  and consequently
  \[
    \int_\tau^\infty \cQ_s \dif s = \infty,
  \]
  which also implies
  \begin{align*}
    \liminf_t \frac 1 L \ln \|u_t\|
    &\geq
    \liminf_t \Phi(u_t) - \Phi(u_\tau) + \Phi(u_\tau)
    =
    \Phi(u_\tau) + \liminf_t  \int_\tau^t \frac \dif {\dif s} \Phi(u_s)\dif s
    \\
    &\geq
    \Phi(u_\tau) + \liminf_t  \int_\tau^t \hgamma \cQ_s \dif s
    = \infty,
  \end{align*}
  thus $\|u_s\|\to \infty$.  On the other hand, if $\inf_s \hcR(u_s) = 0$,
  then there exists $t_1$ so that for all $t\geq t_1$, then
  $\gamma_t > 0$ (also making use of nondecreasing margins \citep{kaifeng_jian_margin}),
  which is only possible if $\|u_s\|\to\infty$,
  and thus moreover we can take $t_2\geq t_1$
  so that additionally $\|u_t\| \geq \ln(n)/\gamma_{t_1}$,
  (which will hold for all $t'>t_2$ by \Cref{fact:gf:norm_growth}),
  and by \Cref{fact:ngamma:ub} and the restriction $L=2$ means
\[
    \frac \dif {\dif t} \ln \|u\|
    = \frac {\sum_i |\ell'_i| \ip{u}{\cs p_i(u)}}{\|u\|^2}
    = \frac {\cQ L \sum_i q_i p_i(u)}{\|u\|^2}
    \leq \cQ L (\gamma_t + \gamma_{t_1}) 
    \leq 2 L \cQ \gamma_t,
  \]
  which after integrating and upper bounding $\gamma_t\leq 1$
  means $\int_{t_2}^\infty\cQ_s\dif s \geq \lim_t \sbr{ \ln \|u_t\| - \ln \|u_{t_2}\|} = \infty$.
  As such, independent of whether $\inf_s \hcR(u_s) = 0$, then still
  $\|u_s\|\to\infty$ and $\int_\tau^\infty \cQ_s \dif s = \infty$.

  This now suffices to complete the proof.  If $\liminf_t \gamma_t \geq \hgamma$,
  then in fact $\lim_t\gamma_t$ is well-defined (by non-decreasing margins and
  $\|u\|\to\infty$)
  and $\lim_t\gamma_t \geq \hgamma>0$, whereby $\limsup_t \hcR(u_t) \leq
  \limsup_t \ell(-\gamma_t\|w_t\|^L) = 0$.  Alternatively,
  suppose contradictorily that $\liminf_t \gamma_t < \hgamma$, and choose
  any $\eps \in (0, \hgamma/4)$ so that $\liminf_t \gamma_t < \hgamma - 3\eps$;
  noting that $\gamma_t$ is monotone once there exists some $\gamma_s>0$,
  then, choosing $t_3$ large enough so that $\|u_t\|^2 \geq \ln(n)/\eps$ for all $t\geq t_3$
  and $\gamma_t < \hgamma - 2\eps$ for all $t\geq t_3$,
  it follows by \Cref{fact:ngamma:ub} that
  \begin{align*}
    0
    &\leq \liminf_t \sbr{ \frac 1 L \ln \|u_t\| - \Phi(u_t) }
    \\
    &\leq \frac 1 L \ln \|u_{t_3}\| - \Phi(u_{t_3})
    +
    \liminf_t \int_{t_3}^t \frac {\dif}{\dif s} \sbr{ \frac 1 L \ln \|u_s\| - \Phi(u_s) }
    \dif s
    \\
    &\leq \frac 1 L \ln \|u_{t_3}\| - \Phi(u_{t_3})
    +
    \liminf_t \int_{t_3}^t \sbr{ \cQ(\hgamma - \eps) - \cQ \hgamma }
    \dif s
    \\
    &\leq \frac 1 L \ln \|u_{t_3}\| - \Phi(u_{t_3})
    +
    \liminf_t \int_{t_3}^t \sbr{ - \eps  \cQ }
    \dif s
    \\
    &=
    -\infty,
  \end{align*}
  a contradiction, and since $\eps\in (0,\hgamma/4)$ was arbitrary,
  it follows that $\liminf \gamma_t \geq \hgamma$.
\end{proof}

\section{Proofs for \Cref{sec:ntk}}

This section contains proofs with a dependence on $\gntk$.

\subsection{SGD proofs}
\label{sec:app:sgd}

The following application of Freedman's inequality is used to obtain the test error bound.

\begin{lemma}[{Nearly identical to \citep[Lemma 4.3]{ziwei_ntk}}]
  \label{fact:freedman}
  Define $\cQ(W) := \bbE_{x,y} |\ell'(p(x,y;W))|$
  and $\cQ_i(W) := |\ell'(p(x_i,y_i;W))|$.
  Then $\sum_{i<t} \sbr{ \cQ(W_i) - \cQ_i(W_i) }$ is a martingale difference sequence,
  and with probability at least $1-\delta$,
  \[
    \sum_{i<t} \cQ(W_i) \leq 4 \sum_{i<t} \cQ_i(W_i) + 4\ln(1/\delta),
  \]
\end{lemma}
\begin{proof}
  This proof is essentially a copy of one due to \citet[Lemma 4.3]{ziwei_ntk};
  that one is stated for the analog of $p_i$ used there, and thus need to be re-checked.

  Let $\cF_i := \{ ((x_j,y_j)) : j < i\}$ denote the $\sigma$-field of all information
  until time $i$, whereby $x_i$ is independent of $\cF_i$, whereas $w_i$ deterministic
  after conditioning on $\cF_i$.
  Consequently, $\bbE\sbr{ \cQ(W_i) - \cQ_i(W_i) | \cF_i } = 0$,
  whereby $\sum_{i<t} \sbr{ \cQ(W_i) - \cQ_i(W_i) }$ is a martingale difference sequence.

  The high probability bound will now follow via a version of Freedman's inequality
  \citep[Lemma 9]{djh_mini_monster}.  To apply this bound, the conditional variances
  must be controlled: noting that $|\ell'(z)| \in [0,1]$,
  then $\cQ(W_i) - \cQ_i(W_i) \leq 1$,
  and
  since $\cQ_i(W_i) \in [0,1]$, then $\cQ_i(W_i)^2 \leq \cQ_i(W_i)$, and thus
  \begin{align*}
    \bbE\sbr{ \del{\cQ(W_i) - \cQ_i(W_i)}^2 \ \big| \ \cF_i }
    &=
    \bbE\sbr{ \cQ_i(W_i)^2 \ \big| \ \cF_i } - \cQ(W_i)^2
    \\
    &\leq
    \bbE\sbr{ \cQ_i(W_i) \ \big| \ \cF_i } - 0
    \\
    &= \cQ(W_i).
  \end{align*}
  As such, by the aforementioned version of Freedman's inequality
  \citep[Lemma 9]{djh_mini_monster},
  \begin{align*}
    \sum_{i<t} \sbr{ \cQ(W_i) - \cQ_i(W_i) }
    &\leq
    (e-2)
    \sum_{i<t} \bbE \sbr{ \del{\cQ(W_i) - \cQ_i(W_i)}^2 \ \big| \ \cF_i }
    + \ln(1/\delta)
    \\
    &\leq
    (e-2)
    \sum_{i<t} \cQ(W_i)
    + \ln(1/\delta),
  \end{align*}
  which rearranges to give
  \[
    (3-e)\sum_{i<t} \cQ(W_i) \leq \sum_{i<t} \cQ_i(W_i) + \ln(1/\delta),
  \]
  which gives the result after multiplying by $4$ and noting $4(3-e)\geq 1$.
\end{proof}

With \Cref{fact:freedman}
and the Gaussian concentration inequalities from \Cref{sec:gaussians} in hand,
a proof of a generalized form of \Cref{fact:sgd} is as follows.

\begin{proof}[Proof of \Cref{fact:sgd}]
Let $(w_j)_{j=1}^m$ be given with corresponding
  $(\bara_j,\barv_j) := \btheta_j := \theta(w_j)/\sqrt{m}$ (whereby $\|\btheta_j\|\leq 2$ by
  construction),
  and define
  \[
    r
    :=
\frac{10\eta\sqrt{m}}{\gamma}
    \leq \frac {\gamma \sqrt{m}}{640},
    \qquad
    R
    := 8r
=\frac{80\eta\sqrt{m}}{\gamma}
    \leq \frac{\gamma \sqrt{m}}{80},
    \qquad
    \barW := r \btheta + W,
  \]
  which implies $r \geq 1$, 
  and $R\geq 1$, and $\eta \leq R/16$.
  Since \Cref{ass:M-D} implies that
  \Cref{ass:M-F} holds for $((x_i,y_i))_{i\leq t}$ with probability $1$,
  for the remainder of the proof, rule out the $7\delta$ failure probability associated
  with the second part of \Cref{fact:gaussians:margin} (which is stated in terms of
  \Cref{ass:M-F} not \Cref{ass:M-D}), whereby simultaneously
  for every $\|W' - W_0\|\leq R$
  \begin{align}
    \min_i \ip{\barW}{\hs p_i(W')}
    &\geq
    \frac {r\gamma\sqrt m}{2} - 160r^2 
    \geq
    \frac {r\gamma \sqrt{m}}{4}
    =
    \frac{\gamma^2 m}{2560}
    \geq
    \ln(t),
    \label{eq:sgd:1}\\
    \min_i \ip{\btheta}{\hs p_i(W')}
    &\geq \gamma\sqrt{m} - \sqrt{32\ln(n/\delta)} - 4R - 4
    \geq \gamma \sqrt{m} - \frac{\gamma \sqrt{m}}{8} - \frac{\gamma \sqrt{m}}{10}
    \geq \frac {\gamma \sqrt{m}}{2},
    \label{eq:sgd:2}
  \end{align}
  and also $\|a_0\|\leq 2$ and $\|V_0\|\leq 2\sqrt{m}$.

  The proof now proceeds as follows.  Let $\tau$ denote the first iteration where
  $\|W_\tau - W_0\|\geq R$,
  whereby $\tau > 0$ and $\max_{s<\tau} \|W_s - W_0\|\leq R$.
  Assume contradictorily that $\tau \leq t$;
  it will be shown that this implies $\|W_\tau - W_0\|\leq R$.

  Consider any iteration $s < \tau$.
  Expanding the square,
  \begin{align*}
    \|W_{s+1} - \barW\|^2
    &=
    \|W_s - \eta \hs \ell_s(W_s) - \barW\|^2
    \\
    &=
    \|W_s -  \barW\|^2
    - 2\eta \ip{\hs\ell_s(W_s)}{W_s - \barW}
    + \eta^2 \enVert{\hs \ell_s(W_s)}^2
    \\
    &=
    \|W_s -  \barW\|^2
    + 2\eta \ell_s'(W_s)\ip{\hs p_s(W_s)}{\barW - W_s}
    + \eta^2 \ell'_s(W_s)^2 \enVert{\hs p_s(W_s)}^2
    .
  \end{align*}
  By convexity, $\|W_s - W_0\|\leq R$, and \cref{eq:sgd:1},
  \begin{align*}
    \ell_s'(W_s)\ip{\hs p_s(W_s)}{\barW - W_s}
    &=
    \ell_s'(W_s)\del{\sbr{\ip{\hs p_s(W_s)}{\barW} - p_s(W_s)} - p_s(W_s)}
    \\
    &\leq
    \ell_s\del{\ip{\hs p_s(W_s)}{\barW} - p_s(W_s)} - \ell_s(W_s)
    \\
    &\leq
    \ln(1+\exp(-\ln(t))) - \ell_s(W_s),
    \\
    &\leq
    \frac 1 t - \ell_s(W_s),
  \end{align*}
  which combined with the preceding display gives
  \begin{align*}
    \|W_{s+1} - \barW\|^2
    &\leq
    \|W_s -  \barW\|^2
    + 2\eta
    \del{\frac 1 t - \ell_s(W_s)}
    + \eta^2 \ell'_s(W_s)^2 \enVert{\hs p_s(W_s)}^2.
  \end{align*}
  Since this inequality holds for any $s<\tau$, then applying
  the summation $\sum_{s<\tau}$ and rearranging gives
  \begin{align*}
    \|W_\tau - \barW\|^2
    + 2\eta \sum_{s<\tau} \ell_s(W_s)
    &\leq
    \|W_0 - \barW\|^2
    + 2\eta
    + \sum_{s<\tau} \eta^2 \ell'_s(W_s)^2 \enVert{\hs p_s(W_s)}^2.
  \end{align*}
  To simplify the last term, using $\|V_0\|\leq 2\sqrt{m}$ and $\|a_0\|\leq 2$
  and $\|W_s - W_0\|\leq R$ gives
  \begin{align*}
    \enVert{\hs p_s(W)}^2
    &= \|\sigma(V_s x_s)\|^2
    + \enVert[2]{\sum_j e_j a_{i,j} \sigma'(v_{i,j}^\T x_s)x_s}^2
    \\
    &\leq \enVert{\sigma(V_s x_s)}^2
    + \enVert{a_s }^2
    \\
    &\leq 2 \|V_s - V_0\|^2 + 2 \|V_0\|^2 + 2 \|a_s - a_0\|^2 + 2\|a_0\|^2
    \\
    &\leq 2R^2 + 8m + 8,
    \\
    &\leq 10m,
  \end{align*}
  and moreover the first term can be simplified via
  \begin{align*}
    \|W_\tau - \barW\|^2
    &= \|W_\tau - W_0\|^2 - 2\ip{W_\tau - W_0}{\barW - W_0} + \|\barW - W_0\|^2
    \\
    &\geq \|W_\tau - W_0\|^2 - 2r \|W_\tau - W_0\| + \|\barW - W_0\|^2,
  \end{align*}
  whereby combining these all gives
  \begin{align*}
    \hspace{2em}&
    \hspace{-2em}
    \|W_\tau - W_0\|^2 - 2r \|W_\tau - W_0\| + \|\barW - W_0\|^2
    + 2\eta \sum_{s<\tau} \ell_s(W_s)
    \\
    &\leq 
    \|W_\tau - \barW\|^2
    + 2\eta \sum_{s<\tau} \ell_s(W_s)
    \\
    &\leq
    \|W_0 - \barW\|^2
    + 2\eta
    + \sum_{s<\tau} \eta^2 \ell'_s(W_s)^2 \enVert{\hs p_s(W_s)}^2
    \\
    &\leq
    \|W_0 - \barW\|^2
    + 2\eta
    + 10 \eta^2 m \sum_{s<\tau} | \ell'_s(W_s) |,
  \end{align*}
  which after canceling and rearranging gives
  \begin{align*}
    \|W_\tau - W_0\|^2 
    + 2\eta \sum_{s<\tau} \ell_s(W_s)
    &\leq
    2r \|W_\tau- W_0\|
    + 2\eta
    + 10 \eta^2 m \sum_{s<\tau} | \ell'_s(W_s)|.
  \end{align*}
  To simplify the last term,
  note by \cref{eq:sgd:2} that
  \begin{align}
    \|W_\tau - W_0\|
    &= \sup_{\|W\|\leq 1} \ip{W}{W_\tau - W_0}
    \notag\\
    &\geq \frac 1 2 \ip{-\bar \theta }{W_\tau - W_0}
    \notag\\
    &= \frac {\eta}{2} \sum_{s<\tau} \ip{-\btheta}{\hs \ell_s(W_s)}
    \notag\\
    &=\frac {\eta}{2} \sum_{s<\tau} |\ell'_s(W_s)| \ip{\btheta}{\hs p_i(W_s)}
    \notag\\
    &\geq\frac {\eta}{2} \sum_{s<\tau} |\ell'_s(W_s)| \frac {\gamma \sqrt{m}}{2},
    \label{eq:sgd:Q}
  \end{align}
  and thus, by the choice of $R$, and since $\|W_\tau - W_0\|\geq 1$ and $\eta \leq R/16$,
  \begin{align*}
    \|W_\tau - W_0\|^2 
    + 2\eta \sum_{s<t} \ell_s(W_s)
    &\leq
    2r \|W_\tau - W_0\|
    + 2\eta
    + \frac {40 \eta \sqrt{m}\|W_\tau - W_0\|}{\gamma}
    \\
    &\leq
    \del{\frac R 4 + \frac R 8 + \frac R 2}
    \|W_\tau - W_0\|.
  \end{align*}
  Dropping the term $2\eta \sum_{s<t} \ell_s(W_s)\geq 0$
  and dividing both sides by $\|W_\tau - W_0\|\geq R > 0$
gives
  \[
    \|W_\tau - W_0\|
    \leq
    \frac R 4
    + \frac R 8
    + \frac R 2
    < R,
  \]
  the desired contradiction, thus $\tau > t$ and all above derivations hold for all
  $s\leq t$.

  To finish the proof, combining \cref{eq:sgd:Q}
  with $\|W_t - W_0\|\leq R = 80 \eta \sqrt{m} / \gamma$ gives
  \begin{align*}
    \sum_{s<t} |\ell'_s(W_s)|
    &\leq
    \frac {4 \|W_t - W_0\|}{\eta \gamma \sqrt{m}}
    \leq
    \frac {320}{ \gamma^2 }.
  \end{align*}
  Lastly, for the generalization bound, defining $\cQ(W) := \bbE_{x,y} |\ell'(p(x,y;W))|$,
  discarding an additional $\delta$ failure probability,
  by \Cref{fact:freedman},
  \[
    \sum_{i<t} \cQ(W_s) \leq 4 \ln(1/\delta) + 4 \sum_{s<t} |\ell'_s(W_s)|
    \leq 4 \ln(1/\delta) + \frac {1280}{\gamma^2}.
  \]
  Since $\1[p_s(W_s) \leq 0] \leq 2 |\ell'_s(W_s)|$,
  the result follows.
\end{proof}

\subsection{GF proofs}
\label{sec:app:gf}

This section culminates in the proof of \Cref{fact:gf:margins},
which is immediate once \Cref{fact:gf,fact:gf:large_margin} are established.

Before proceeding with the main proofs, the following technical lemma is used
to convert a bound on $\ell'$ to a bound on $\ell$.

\begin{lemma}\label{fact:loss:derivative}
  For $\ell\in\{\llog, \lexp\}$,
  then $|\ell'(z)| \leq 1/8$ implies $\ell(z) \leq 2|\ell'(z)|$.
\end{lemma}
\begin{proof}
  If $\ell = \lexp$, then $\ell' = -\ell$, and thus $\ell(z) \leq 2|\ell'(z)|$ automatically.
  If $\ell(z) = \llog$, the logistic loss,
  then $|\ell'(z)|\leq 1/8$ implies $z \geq 2$.
  By the concavity of $\ln(\cdot)$,
  for any $z\geq 2$, since $1+e^{-z} \leq 7/6$,
  then
  \[
    \ell(z) = \ln(1+e^{-z})\leq e^{-z} \leq \frac {(7/6) e^{-z}}{1+e^{-z}} \leq 2 |\ell'(z)|,
  \]
  thus completing the proof.
\end{proof}

Next comes the proof of \Cref{fact:gf},
which follows the same proof plan as \Cref{fact:sgd}.

\begin{proof}[Proof of \Cref{fact:gf}]
  This proof is basically identical to the SGD in \Cref{fact:sgd}.
  Despite this, proceeding with amnesia,
  let rows $(w_j)_{j=1}^m$ of $W_0$ be given with corresponding
  $(\bara_j,\barv_j) := \btheta_j := \theta(w_j)/\sqrt{m}$ (whereby $\|\btheta_j\|\leq 2$ by
  construction),
  and define
  \[
    r := \frac{\gntk \sqrt{m}}{640},
    \qquad
    R := 8r = \frac{\gntk \sqrt{m}}{80},
    \qquad
    \barW := r \btheta + W,
  \]
  with immediate consequences that $r\geq 1$ and $R\geq 8$.
  For the remainder of the proof, rule out the $7\delta$ failure probability associated
  with the second part of \Cref{fact:gaussians:margin}, whereby simultaneously
  for every $\|W' - W_0\|\leq R$,
  \begin{align}
    \min_i \ip{\barW}{\hs p_i(W')}
    &\geq
    \frac {r\gntk\sqrt m}{2} - 160r^2 
    \geq
    \frac {r\gntk \sqrt{m}}{4}
    =
    \frac{\gntk^2 m}{2560}
    \geq
    \ln(t),
    \label{eq:gf:1}\\
    \min_i \ip{\btheta}{\hs p_i(W')}
    &\geq \gntk\sqrt{m} - \sqrt{32\ln(n/\delta)} - 4R - 4
    \geq \gntk \sqrt{m} - \frac{\gntk \sqrt{m}}{8} - \frac{\gntk \sqrt{m}}{10}
    \geq \frac {\gntk \sqrt{m}}{2}.
    \label{eq:gf:2}
  \end{align}
  The proof now proceeds as follows.  Let $\tau$ denote the earliest time such that
  $\|W_\tau - W_0\|= R$;
  since $W_s$ traces out a continuous curve and since $R > 0 = \|W_0 - W_0\|$,
  this quantity is well-defined.
  As a consequence of the definition, $\sup_{s<\tau} \|W_s - W_0\|\leq R$.
  Assume contradictorily that $\tau \leq t$;
  it will be shown that this implies $\|W_\tau - W_0\|<R$.

  By the fundamental theorem of calculus (and the chain rule for Clarke differentials),
  convexity of $\ell$,
  and since $\|W_s - W_0\|\leq R$ holds for $s\in[0,\tau)$,
  which implies \cref{eq:gf:1} holds,
  \begin{align*}
    \|W_\tau - \barW\|^2
    - \|W_0 - \barW\|^2
    &=
    \int_0^\tau \frac {\dif}{\dif s} \|W_s - \barW\|^2 \dif s
    \\
    &=
    \int_0^\tau 2 \ip{\dot W_s}{W_s - \barW} \dif s
    \\
    &=
    \frac{2}{n} 
    \int_0^\tau \sum_i \ell'_i(W_s) \ip{\cs p_i(W_s)}{W_s - \barW} \dif s
    \\
    &=
    \frac{2}{n} 
    \int_0^\tau \sum_i
    \ell_i'(W_s)\del{\sbr{\ip{\hs p_i(W_s)}{\barW} - p_i(W_s)} - p_i(W_s)}
    \dif s
    \\
    &\leq
    \frac{2}{n} 
    \int_0^\tau \sum_i
    \del{
    \ell_i\del{\ip{\hs p_i(W_s)}{\barW} - p_i(W_s)} - \ell_i(W_s)}
    \dif s
    \\
    &\leq
    \frac{2}{n} 
    \int_0^\tau \sum_i \del{ \frac 1 t - \ell_i(W_s) } \dif s
    \\
    &\leq 2 - 2 \int_0^\tau \hcR(W_s) \dif s.
  \end{align*}
  To simplify the left hand side,
  \[
    \|W_\tau - \barW\|^2  - \|W_0 - \barW\|^2
    = \|W_\tau - W_0\|^2 - 2\ip{W_\tau - W_0}{\barW - W_0}
    \geq \|W_\tau - W_0\|^2 - 2r \|W_\tau - W_0\|,
  \]
  which after combining, rearranging, and using $r\geq 1$ and $\|W_\tau-W_0\|\geq R \geq 1$
  gives
  \[
    \|W_\tau - W_0\|^2
    + 2 \int_0^\tau \cR(W_s) \dif s
    \leq 2 + 2r \|W_\tau - W_0\|
    \leq 4r \|W_\tau - W_0\|,
  \]
  which implies
  \[
    \|W_\tau - W_0\| \leq 2r = \frac R 2 < R,
  \]
  a contradiction since $W_\tau$ is well-defined as the earliest time with
  $\|W_\tau - W_0\| = R$, which thus contradicts $\tau \leq t$.
  As such, $\tau \geq t$, and all of the preceding inequalities follows with $\tau$
  replaced by $t$.

  To obtain an error bound, similarly to the key perceptron argument before,
  using \cref{eq:gf:2},
  \begin{align*}
    \|W_t - W_0\|
    &= \sup_{\|W\|\leq 1} \ip{W}{W_t - W_0}
    \\
    &\geq \frac 1 2 \ip{-\btheta}{W_t - W_0}
    \\
    &=\frac 1 2 \ip{-\btheta}{\int_0^t \dot a_s \dif s}
    \\
    &= \frac 1 {2n} \int_0^t \sum_i |\ell'_i(W_s)|\ip{\btheta}{\hs p_i(W_s)}\dif s
    \\
    &\geq
    \frac {\gntk \sqrt{m}}{4n} \int_0^t \sum_i |\ell'_i(W_s)| \dif s,
  \end{align*}
  which implies
  \[
    \frac 1 n
    \int_0^t \sum_i |\ell'_i(W_s)| \dif s
    \leq
    \frac{ 4\|W_t - W_0\| }{\gntk \sqrt{m}}
    \leq
    \frac{1}{20},
  \]
  and in particular
  \[
    \inf_{s \in [0,t] }
    \frac 1 n
    \sum_i |\ell'_i(W_s)|
    \leq
    \frac 1 {tn} \int_0^t \sum_i |\ell'_i(W_s)| \dif s
    \leq
    \frac{ 4\|W_t - W_0\| }{t \gntk \sqrt{m}}
    \leq
    \frac{1}{20t}
  \]
  and so there exists $k\in[0,t]$ with
  \[
    \frac 1 n \sum_i |\ell'_i(W_{k})| \leq \frac {1}{10t}.
  \]
  Since this also implies $\max_i |\ell'_i(W_k)| \leq n / (10t)\leq 1/10$,
  it follows by \Cref{fact:loss:derivative} that $\hcR(W_k)\leq 1/(5t)$,
  and the claim also holds for $t'\geq t$ since the empirical risk is nonincreasing
  with gradient flow.
\end{proof}

Next, the proof of the Rademacher complexity bound used for all GF sample complexities.

\begin{proof}[Proof of \Cref{fact:rad}]
  For any $(a,v)\in\R^{d+1}$ and any $x$,
  recalling the notation defining $\ta := \sgn(a)$ and $\tv := v/\|v\|$,
  \[
    a \sigma(v^\T x) = \|av\| \ta \sigma(\tv^\T x) \leq \frac {\|(a,v)\|^2}{2} \ta \sigma(\tv^\T x),
  \]
  and therefore, letting
  \[
    \sconv(S) := \cbr{ \sum_{j=1}^m p_j u_j \  : \  m\geq 0, p\in \R^m, \|p\|_1\leq 1, u_j \in S}
  \]
  denote the symmetrized convex hull as used throughout Rademacher complexity \citep{shai_shai_book},
  then
  \begin{align*}
    \cF
    &:= \cbr{
      x\mapsto \frac {1}{\|W\|^2} \sum_{j} a_j \sigma(v_j^\T x)
      \ : \  m\geq 0, W \in \R^{m\times (d+1)}
    }
    \\
    &= \cbr{
      x\mapsto \frac {1}{\|W\|^2} \sum_{j} \|a_jv_j\| \ta \sigma(\tv_j^\T x)
      \ : \  m\geq 0, W \in \R^{m\times (d+1)}
    }
    \\
    &\subseteq \cbr{
      x\mapsto \sum_j \frac{ \|(a_j, v_j)\|^2}{2\|W\|^2} \ta \sigma(\tv_j^\T x)
      \ : \  m\geq 0, W \in \R^{m\times (d+1)}
    }
    \\
    &\subseteq \cbr{
      x\mapsto \sum_j p_j \sigma(u_j^\T x)
      \ : \  m\geq 0, p \in \R^m, \|p\|_1\leq \frac 1 2, \| u_j \|_2 = 1
    }
    \\
    &=
    \frac 1 2 \sconv\del{\cbr{ x \mapsto \sigma(v^\T x) : \|v\|_2 = 1}}.
  \end{align*}
  As such, by standard rules of Rademacher complexity \citep{shai_shai_book},
  \[
    \Rad(\cF) \leq \frac 1 2 \Rad\del{\cbr{x\mapsto \sigma(v^\T x) : \|v\|_2 = 1}}
    \leq
    \frac {1}{2\sqrt n},
  \]
  and thus, by a refined margin bound for Rademacher complexity \citep[Theorem 5]{nati_smoothness},
  with probability at least $1-\delta$, simultaneously for all $\ggl$ and all $m$,
  every $W\in\R^{m\times (d+1)}$ with $\ngamma(W) \geq \ggl$ satisfies
  \begin{align*}
    \Pr[ p(x,y;W) \leq 0]
    =
    \cO\del{
      \frac {\ln(n)^3}{\ggl^2} \Rad(\cF)^2 + \frac{\ln \ln \frac 1 \ggl + \ln \frac 1 \delta}{n}
    }
    =
    \cO\del{
      \frac {\ln(n)^3}{n\ggl^2 }
      +
      \frac{\ln \frac 1 \delta}{n}
    },
  \end{align*}
  and to finish, instantiating this bound with $\ngamma(W)$ gives the desired form.
\end{proof}

The proof of \Cref{fact:gf:large_margin} now follows.

\begin{proof}[Proof of \Cref{fact:gf:large_margin}]
  By the second part of \Cref{fact:gaussians:margin}, with probability at least $1-7\delta$,
  simultaneously $\|a\|\leq 2$, and $\|V\|\leq 2\sqrt m$,
  and for any $\|W' - W_0\|\leq R$, then
  \[
    \min_i \ip{\btheta}{\hs p_i(W')}
    \geq \gntk \sqrt{m} - \sbr{\sqrt{32 \ln(n/\delta)} + 8R + 4}
    \geq \frac {\gntk \sqrt m} 2,
  \]
  where $\btheta_j := \theta(w_j)/\sqrt{m}$ as usual, and $\|\btheta\|\leq 2$;
  for the remainder of the proof, suppose these bounds, and discard the corresponding
  $7\delta$ failure probability.  Moreover, for any $W'$
  with $\hcR(W') < \ell(0)/n$ and 
  $\|W' - W_0\|\leq R$, as a consequence of the preceding lower bound
  and also the property $\sum_i q_i(W') \geq 1$
  \citet[Lemma 5.4, first part, which does not depend on linear predictors]{refined_pd},
  \begin{align*}
    \enVert{ \cs \tgamma(W') }
    &= \sup_{\|W\|\leq 1} \ip{W}{\cs \tgamma(W')}
    \\
    &\geq \frac 1 2 \ip{\btheta}{\sum_i q_i \cs p_i(W')}
    \\
    &= \frac 1 2 \sum_i q_i \ip{\btheta}{\cs p_i(W')}
    \\
    &\geq \frac {\gntk \sqrt{m}}{4} \sum_i q_i
    \\
    &\geq \frac {\gntk \sqrt{m}}{4}.
  \end{align*}

  Now consider the given $W_\tau$ with $\hcR(W_\tau) < \ell(0)/n$
  and $\|W_\tau - W_0\|\leq R/2$.
  Since $s\mapsto W_s$ traces out a continuous curve and since norms grow monotonically
  and unboundedly after time $\tau$ (cf. \Cref{fact:gf:norm_growth}),
  then there exists a unique time $r$ with $\|W_t - W_0\| = R$.
  Furthermore, since $\hcR$ is nonincreasing throughout gradient flow, then
  $\hcR(W_s) < \ell(0)/n$ holds for all $s \in [\tau, t]$.
  Then
  \begin{align*}
    \tgamma(W_t) - \tgamma(W_\tau)
    &=
    \int_\tau^t \ip{\cs \tgamma(W_s)}{\dot W_s}\dif s
    \\
    &=
    \int_\tau^t \|\cs \tgamma(W_s)\| \cdot \|\dot W_s\| \dif s
    \\
    &\geq \frac {\gntk \sqrt{m}}{4} 
    \int_\tau^t \|\dot W_s\| \dif s
    \\
    &\geq \frac {\gntk \sqrt{m}}{4} 
    \enVert[2]{\int_\tau^t \dot W_s} \dif s
    \\
    &= \frac {\gntk \sqrt{m}}{4} 
    \|W_t - W_\tau\| \dif s
    \\
    &\geq \frac {\gntk R \sqrt{m}}{8} 
    \\
    &\geq \frac {\gntk^2 m}{256}.
  \end{align*}
  Since $\|W_0\| \leq 3\sqrt{m}$,
  thus $\|W_t\| \leq 3\sqrt{m} + \gntk \sqrt{m}/32 \leq 4\sqrt{m}$,
  and the normalized margin satisfies
  \[
    \ngamma(W_t)
    \geq
    \frac {\tgamma(W_\tau)}{\|W_t\|^2}
    + \frac 1 {\|W_t\|^2} \int_{\tau}^t \frac {\dif}{\dif s} \tgamma(W_s) \dif s
    \geq
    0
    + \frac {\gntk^2 m / 256} {16m}
    = \frac {\gntk^2}{4096}.
  \]
  Furthermore, it holds that $\ngamma(W_s) \geq \ngamma(W_t)$ for all $s\geq t$
  \citep{kaifeng_jian_margin},
  which completes the proof for $W_t$ under the standard parameterization.

  Now consider the rebalanced parameters $\hatW_t := (a_t / \sqrt{\gntk}, V_t\sqrt{\gntk})$;
  since $m \geq 256 / \gntk^2$, which means $16 \leq \gntk \sqrt{m}$, then 
  \begin{align*}
    \|a_t\| &\leq \|a_0\| + \|a_t - a_0\| \leq 2 + R \leq \frac {\gntk \sqrt{m}}{8}
    + \frac{\gntk \sqrt{m}}{32} \leq \frac {\gntk \sqrt{m}}{4},
    \\
    \|V_t\| &\leq \|V_0\| + \|V_t - V_0\| \leq 2\sqrt m + R \leq 3\sqrt m,
  \end{align*}
  then the rebalanced parameters satisfy
  \[
    \|\hatW_t\| \leq
    \|a_t/\sqrt{\gntk}\| + \|V_t\sqrt{\gntk}\|
    \leq \frac {\sqrt{\gntk m}}{4} + 3\sqrt{\gntk m}
    \leq 4\sqrt{\gntk m},
  \]
  and thus, for any $(x,y)$, since
  \[
    p(x,y;w)
    = \sum_j a_j \sigma(v_j^\T x) = \sum_j \frac {a_j}{\sqrt{\gntk}} \sigma(\sqrt{\gntk}v_j^\T x)
    = p(x,y;\hatW),
  \]
  then
  \[
    \ngamma(\hatW)
    =
    \min_i \frac {p_i(\hatW)}{\|\hatW\|^2}
    =
    \min_i \frac {p_i(w)}{\|\hatW\|^2}
    \geq
    \frac {\tgamma^2 m/256}{16 \gntk m}
    \geq
    \frac {\gntk}{4096},
  \]
  which completes the proof.
\end{proof}

Thanks to \Cref{fact:gf,fact:gf:large_margin}, the proof of \Cref{fact:gf:margins} is now immediate.

\begin{proof}[Proof of \Cref{fact:gf:margins}]
  As in the statement,
  define $R := \gntk \sqrt{m} / 32$, and note that \Cref{ass:M-F} holds almost surely
  for any finite sample due to \Cref{ass:M-D}.
  The analysis now uses two stages.
  The first stage is handled by \Cref{fact:gf} run until time $\tau := n$,
  whereby, with probability at least $1-7\delta$,
  \[
    \hcR(W_\tau) \leq \frac 1 {5n} < \frac {\ell(0)}{n},
    \qquad
    \|W_\tau - W_0\| \leq \frac {\gntk \sqrt{m}}{80} \leq \frac {R}{2}.
  \]
  The second stage now follows from 
  \Cref{fact:gf:large_margin}:
  since $W_\tau$ as above satisfies all the conditions of \Cref{fact:gf:large_margin},
  there exists $W_t$ with $\|W_t-W_0\|= R$, and $\ngamma(W_s) \geq \gntk^2/4096$
  for all $s\geq t$, and $\ngamma(\hatW_t) \geq \gntk / 4096$ for the rebalanced
  iterates $\hatW_t = (a_t/\sqrt{\gntk}, V_t\sqrt{\gntk})$.
  The first generalization bound now follows by \Cref{fact:rad}.
  The second generalization bound uses $p(x,y;W_t) = p(x,y;\hatW_t)$ for all $(x,y)$,
  whereby $\Pr[p(x,y;W_t) \leq 0] = \Pr[p(x,y;\hatW_t) \leq 0]$,
  and the earlier margin-based generalization bound can be invoked with
  the improved margin of $\hatW_t$.
\end{proof}

Lastly, here are the details from the construction leading to \Cref{fact:gf:kkt} which
give a setting where GF converges to parameters with higher margin than the maximum margin
linear predictor.

\begin{proof}[Proof of \Cref{fact:gf:kkt}]
  Let $u$ be the maximum margin linear separator for $((x_i,y_i))_{i=1}^n$; necessarily,
  there exists at least one support vector in each of $S_1$ and $S_2$, since
  otherwise $u$ is also the maximum margin linear predictor over just one of the sets,
  but then the condition
  \begin{equation}
    \min\{\ip{x_i}{x_j} : x_i \in S_1,x_j\in S_2\}\leq - \frac 1 {\sqrt{2}}
    \label{eq:kkt:cond}
  \end{equation}
  implies that $u$ points away from and is incorrect on the set with no support vectors.

  As such, let $v_1\in S_1$ and $v_2\in S_2$ be support vectors, and consider the addition
  of a single data point $(x',+1)$ from a parameterized family of points 
  $\{ z_\alpha : \alpha \in [0,1] \}$, defined as follows.
  Define $v_0 := (I - v_1v_1^\T)u / \|(I - v_1v_1^\T)u\|$,
  which is orthogonal to $-v_1$ by construction,
  and consider the geodesic between
  $v_0$ and $-v_1$:
  \[
    z_\alpha := \alpha v_0 - \sqrt{1-\alpha^2} v_1,
  \]
  which satisfies $\|z_\alpha\|=1$ by construction by orthogonality, meaning
  \[
    \|z_\alpha\|^2 = \alpha^2 - 2 \alpha\sqrt{1-\alpha^2}\ip{v_0}{v_1} + (1-\alpha^2) = 1.
  \]

  In order to pick a specific point along the geodesic,
  here are a few observations.
  \begin{enumerate}
    \item
      First note that $-v_2$ is a good predictor for $S_1$: \cref{eq:kkt:cond}
      implies
      \[
        \min_{x \in S_1} \ip{-v_2}{x}
        =
        -\max_{x \in S_1} \ip{v_2}{x}
        \geq \frac {1}{\sqrt 2}.
      \]
      It follows similarly that $-v_1$ is a good predictor for $S_2$:
      analogously, $\min_{x\in S_2} \ip{-v_1}{x} \geq 1/\sqrt{2}$.

    \item
      It is also the case that $-v_1$ is a good predictor for every $z_\alpha$
      with $\alpha \leq 1/\sqrt{2}$:
      \[
        \ip{-v_1}{z_\alpha}
        = 
        \sqrt{1-\alpha^2}
        \geq \frac 1 {\sqrt{2}}
        .
      \]

    \item
      Now define a 2-ReLU predictor
      $f(x) := (\sigma(-v_1^\T x) + \sigma(-v_2^\T x))/2$.
      As a consequence of the two preceding points, for any $x \in S_1 \cup S_2 \cup \{z_\alpha\}$,
      then $f(x) \geq 1/\sqrt 2$ so long as $\alpha \leq 1/\sqrt{2}$.
      As such, by \Cref{fact:gamma:cones}, \Cref{ass:M-F} is satisfied with
      $\gamma_1 \geq 1 / (64\sqrt{d})$.

    \item
      Lastly, as $\alpha \to 0$, then $z_\alpha \to -v_1$, but the resulting set of points $S_1\cup S_2 \cup\{z_0\}$
      is linearly separably only with margin at most zero (if it is linearly separable at all).
      Consequently, there exist choices of $\alpha_0 \in [0,1]$ so that the resulting maximum margin linear predictor
      over $S_1 \cup S_2 \cup \{z_\alpha\}$ has arbitrarily small yet still positive margins,
      and for concreteness choose some $\alpha_0$ so that the resulting linear separability margin
      $\gamma$ satisfies $\gamma \in (0, c/(2^{15} d))$,
      where $c$ is the positive constant in \Cref{fact:gf}.
  \end{enumerate}

  As a consequence of the last two bullets, using label $+1$ and inputs $S_1 \cup S_2 \cup \{z_{\alpha_0}\}$,
  the maximum margin linear predictor has margin $\gamma$,
  \Cref{ass:M-F} is satisfied with margin at least $\gamma_1$,
  but most importantly, by \Cref{fact:gf} (with $m_0$ given by the width lower bound required there),
  GF will achieve a margin $\gamma_2 \geq c\gamma_1^2/4 \geq 2\gamma$.

  It only remains to argue that the linear predictor itself is a KKT point (for the margin objective),
  but this is direct and essentially from prior work \citep{kaifeng_jian_margin}:
  e.g., taking all outer weights to be equal and positive, and all inner weights
  to be equal to $u$ and of equal magnitude and equal to the outer layer magnitude, and then taking
  all these balanced norms to infinity, it can be checked that the gradient and parameter alignment
  conditions are asymptotically satisfied (indeed, the ReLU can be ignored since all nodes have
  positive inner product with all examples), which implies convergence to a KKT point
  \citep[Appendix C]{kaifeng_jian_margin}.
\end{proof}

\section{Proofs for \Cref{sec:beyond}}

This section develops the proofs of \Cref{fact:nc} and \Cref{fact:gl}.  Before proceeding,
here is a quick sampling bound which implies there exist ReLUs pointing in good directions
at initialization, which is the source of the exponentially large widths in the
two statements.

\begin{lemma}\label{fact:cap_sampling}
  Let $(\beta_1,\ldots,\beta_r)$ be given with $\|\beta_k\| = 1$,
  and suppose $(v_j)_{j=1}^m$ are sampled with $v_j\sim \cN_d/\sqrt{d}$,
  with corresponding normalized weights $\tv_j$.
  If
  \[
    m \geq 2 \del{\frac 2 \eps}^d \ln \frac r \delta,
  \]
  then $\max_k \min_j \|\tv_j - \beta_k\|  \leq \eps$,
  alternatively $\min_k \max_j \tv_j^\T \beta_k \geq 1 - \frac {\eps^2}{2}$.
\end{lemma}
\begin{proof}
  By standard sampling estimates \citep[Lemma 2.3]{ball_geometry}, for any fixed
  $k$ and $j$, then
  \[
    \Pr[ \|\tv_j - \beta_k\| \leq \eps ] \geq \frac 1 2 \del{\frac \eps 2}^{d-1},
  \]
  and since all $(v_j)_{j=1}^m$ are iid,
  \begin{align*}
    \Pr[\exists j\centerdot \|\tv_j - \beta_k\| \leq \eps]
    &=
    1 - \Pr[\forall j\centerdot \|\tv_j - \beta_k\| > \eps]
    =
    1 - \del{1 - \Pr[\|\tv_1 - \beta_k\| \leq \eps]}^m
    \\
    &\geq
    1 - \exp\del{ - \frac m 2 \del{ \frac \eps 2 }^{d-1} }
    \geq 1-\frac \delta r,
  \end{align*}
  and union bounding over all $(\beta_k)_{k=1}^r$ gives
  $\max_k \min_j \|\tv_j - \beta_k\|  \leq \eps$.
  To finish, the inner product form comes by noting
  $\|\tv_j - \beta_k\|^2 = 2 - 2 \tv_j^\T \beta_k$
  and rearranging.
\end{proof}

First comes the proof of \Cref{fact:nc}, whose entirety is the construction of a potential
$\Phi$ and a verification that it satisfies the conditions in \Cref{fact:abstract:Phi}.

\begin{proof}[Proof of \Cref{fact:nc}]
  To start the construction of $\phi$,
  first define per-weight
  potentials $\phi_{k,j}$ which track rotation of mass towards each $\beta_k$:
  \[
    \phi_{k,j}
    := \phi_k(w_j)
    := \phi\del{
      \talpha_k a_j \sigma(v_j^\T \beta_k) \geq (1-\eps) \|a_j v_j\|
    }.
  \]
  The goal will be to show that $\frac {\dif}{\dif t} \phi_{k,j}$
  is monotone nondecreasing, which shows that weights get trapped pointing towards
  each $\beta_k$ once sufficiently close.  As such,
  to develop $\frac {\dif}{\dif t}\phi_{k,j}$,
  note (using $\ta_j := \sgn(a_j)$ and $\tv_j := v_j/\|v_j\|$),
  \begin{align*}
    \frac {\dif}{\dif t} \talpha_k a_j\sigma(v_j^\T \beta_k)
        &= \talpha_k\sbr{ \dot a_j \sigma(v_j^\T \beta_k) + a_j \frac \dif {\dif t} \sigma(v_j^\T \beta_k) }
        \\
        &= - \talpha_k
        \sum_i \ell'_i y_i \sbr{ \sigma(v_j^\T x_i) \sigma(v_j^\T \beta_k) + a_j^2\sigma'(v_j^\T \beta_k) x_i^\T \beta_k },
        \\
        &= - \talpha_k
        \sum_i \ell'_i y_i \sbr{\|v_j\|^2 \sigma(\tv_j^\T x_i) \sigma(\tv_j^\T \beta_k) + a_j^2\sigma'(v_j^\T \beta_k) x_i^\T \beta_k },
        \\
        \frac {\dif}{\dif t} \|a_jv_j\|
        &=
        \frac {\dif}{\dif t} \ip{a_jv_j}{a_jv_j}^{1/2}
        \\
        &=
        \frac {2 \ip{a_jv_j}{\dot a_j v_j + a_j \dot v_j}} {2 \ip{a_jv_j}{a_jv_j}^{1/2}}
        \\
        &=
        \frac {\ip{a_jv_j}{-\sum_i \ell'_i y_i \sbr{ v_j\sigma(v_j^\T x_i) + a_j^2 \sigma'(v_j^\T x_i) x_i}}} {\|a_jv_j\|}
        \\
        &=
        \frac {-\sum_i \ell'_i p_i(w_j)  \|w_j\|^2}{\|a_jv_j\|}
        \\
        &=
        -\sum_i \ell'_i \ta_j \sigma(\tv_j^\T x_i) \|w_j\|^2,
        \\
        \del{\frac \dif {\dif t} \phi_j}
        &= \frac {\dif}{\dif t} \sbr{\talpha_k a_j \sigma(v_j^\T \beta_k) - (1-\eps) \|a_jv_j\|}
        &\hspace{-1in}\textup{when $\phi_j \in (0,1)$}
        \\
        &=- \sum_i \ell'_i y_i\Big[
          \|v_j\|^2 \ta_j \sigma(\tv_j^\T x_i)
          \del{ \talpha_k \ta_j \sigma(\tv_j^\T \beta_k) - (1-\eps)}
          \\
        &\qquad\qquad
        + a_j^2 \del{\talpha_k\sigma'(v_j^\T \beta_k) \beta_k^\T x_i
        - (1-\eps) \ta_j \sigma(\tv_j^\T x_i) }
      \Big]
      .
  \end{align*}
  Analyzing the last two terms separately,
  the first (the coefficient to $\|v\|^2$
  is nonnegative since the term in parentheses is a rescaling of $\phi_j$:
  \begin{align*}
    \del{ \talpha_k \ta_j \sigma(\tv_j^\T \beta_k) - (1-\eps)}
        &=
        \frac {\phi_j}{\|a_jv_j\|}
> 0,
      \end{align*}
      and moreover this holds for any choice of $\eps>0$.
      The second term (the coefficient of $a_j^2$) is more complicated;
      to start, fixing any example $(x_i,y_i$ and writing $z_i := c\beta_k + z_\perp$,
      where necessarily $c \geq \gnc$, note (using $\tu_j := \ta_j \tv_j$)
      \begin{align*}
        \ip{\frac {z_\perp}{\|z_\perp\|}}{\tu_j}^2
        &\leq
        \enVert{(I-\beta_k\beta_k^\T)\tu_j}^2
        = 1 - \ip{\beta_k}{\tu_j}^2
        \leq 1 - (1-\eps)^2
        = 2\eps - \eps^2
        \leq 2\eps,
      \end{align*}
      and thus, since $\phi_j \in (0,1)$ whereby $\sigma(v_j^\T x_i) = v_j^\T x_i$,
      the nonnegativity of the second term follows from the narrowness of the 
      around $\beta_k$ (cf. \Cref{ass:NC}):
      \begin{align*}
        \talpha_k \beta_k^\T x_i y_i
        - (1-\eps) \ta_j \tv_j^\T x_i y_i
          &\geq
          c
- (1-\eps) \ip{\ta_j \tv_j}{c\beta_k + z_{\perp}}
          \\
          &\geq
          c\eps - (1-\eps) \|z_{\perp}\| \sqrt{2\eps}
          \\
          &\geq
          c\eps - (1-\eps) \gnc \eps
          \\
          &\geq 0,
      \end{align*}
      meaning $\dif \phi_{k,j} /\dif t \geq 0$ for every pair $(k,j)$.

      With this in hand, define the overall potential as
      \[
        \Phi(w) :=
        \frac 1 4 \sum_k |\alpha_k| \ln \sum_j \phi_{k,j} \|a_j v_j\|,
      \]
      whereby \begin{align*}
        \frac {\dif}{\dif t}
        \Phi(w)
        & =
        \frac 1 4
        \sum_k |\alpha_k| \frac {\sum_j \sbr{
            \phi_{k,j} \frac {\dif}{\dif t}\|a_jv_j\|
            + 
            \|a_jv_j\| \frac {\dif}{\dif t} \phi_{k,j} 
        }}{\sum_j \phi_{k,j} \|a_j v_j\|}
        \\
        & \geq
        \frac 1 4
        \sum_k |\alpha_k| \frac {-\sum_i \ell'_i y_i \sum_j
          \phi_{k,j} \ta_j \sigma(\tv_j^\T x_i) \|w_j\|^2
        }{\sum_j \phi_{k,j} \|a_j v_j\|}
        \\
        &=
        \frac 1 4
        \cQ
        \sum_k 
        |\alpha_k|
        \sum_{i \in S_k} 
        q_i
        \frac {y_i \sum_j
          \phi_{k,j} \talpha_k \sigma\del{(\tv_j-\beta_k + \beta_k)^\T x_i} \|w_j\|^2
        }{\sum_j \phi_{k,j} \|a_j v_j\|}
        \\
        &\geq
        \frac 1 4
        \cQ
        \sum_k 
        \alpha_k
        \sum_{i \in S_k} 
        q_i
        \frac {y_i \sum_j
          \phi_{k,j} \del{\gnc - \eps} 2 \|a_j v_j\|
        }{\sum_j \phi_{k,j} \|a_j v_j\|}
        \\
        &\geq
        \cQ
        \del{
          \frac{\ggl - \eps}{r}
        }.
\end{align*}
Written another way, this establishes $\dif \Phi(W_t) /\dif t \geq \cQ\hgamma$
with $\hgamma = (\gnc-\eps)/k$, which is one of the conditions needed in
\Cref{fact:abstract:Phi}.  The other properties meanwhile are direct:
\[
  \Phi(W_t)
  \leq
  \frac 1 4
  \sum_k |\alpha_k| \ln \sum_j \phi_{j,k} \|w_j\|^2
  \leq
  \frac 1 4
  \sum_k |\alpha_k| \ln \|W_t\|^2
  =
  \frac 1 2 \ln \|W_t\|,
\]
and $\Phi(W_0) > -\infty$ due to random initialization
(cf. \Cref{fact:cap_sampling}),
which allows the application of \Cref{fact:abstract:Phi} and \Cref{fact:rad}
and completes the proof.
\end{proof}

To close, the proof of \Cref{fact:gl}.

\begin{proof}[Proof of \Cref{fact:gl}]
  Throughout the proof, use $W=((a_j,b_k))_{j=1}^m$ to denote the full collection
  of parameters, even in this scalar parameter setting.

  To start, note that $a_k^2 = b_k^2$ for all times $t$; this follows directly,
  from the initial condition $a_k(0)^2 = b_k(0)^2 = 1 / \sqrt{m}$,
  since
  \begin{align*}
    a_k(t)^2 - b_k(t)^2
    &= a_k(t)^2 - a_k(0)^2 - b_k(t)^2 + b_k(0)^2
    \\
    &=
    \int_0^t \del{a_k\dot a_k - b_k\dot b_k}\dif s
    \\
    &=
    \int_0^t \sum_i |\ell'_i| \del{a_k\sigma(b_k v_k^\T x_i)
    - a_k \sigma'(b_k v_k^\T x_i)v_k^\T x_i}\dif s
    \\
    &= 0.
  \end{align*}
  This also implies that $a_k^2 + b_k^2 = 2a_k^2 = 2|a_k|\cdot|b_k|$.

  For each $\beta_k$, choose $j$
  so that $\| \tb_j \tv_j - \beta_k\| \leq \epsilon = \ggl/2$ and $\ta_j = \sgn(\alpha_k)$;
  this holds with probability at least $1-3\delta$ over the draw of $W$ due to the choice
  of $m$, first by noting that with probability at least $1-\delta$, there are at least $m/4$
  positive $a_j$ and $m/4$ negative $a_j$, and then with probability $1-2\delta$
  by applying \Cref{fact:cap_sampling} to $(\tb_j v_j)_{j=1}^m$ (which are equivalent
  in distribution to $(v_j)_{j=1}^m$) for each choice of output sign.
  For the rest of the proof, reorder the weights $((a_j,b_j,v_j))_{j=1}^m$
  so that each $(\alpha_k,\beta_k)$ is associated with $(a_k,b_k,v_k)$.

  Now consider the potential function
  \[
    \Phi(W) := \frac 1 4 \sum_{k=1}^r |\alpha_k| \ln \del{ a_k^2 + b_k^2 }.
  \]
  The time derivative of this potential can be lower bounded in terms of the
  reference margin:
  \begin{align*}
\frac \dif {\dif t} \Phi(W)
    &=
    \sum_k
    \sum_i |\ell'_i|
    y_i 
    |\alpha_k|
    \frac{a_k \sigma(b_k v_k^\T x_i)}{a_k^2 + b_k^2}
    \\
    &=
    \cQ
    \sum_k
    \sum_i q_i
    y_i 
    \alpha_k
    \frac{| a_k | \|b_k v_k\| \sigma(\tb_k \tv_k^\T x_i)}{2|a_k|\cdot|b_k|}
    \\
    &=
    \cQ
    \sum_k
    \sum_i q_i
    y_i 
    \alpha_k
    \sigma\del{(\tb_k \tv_k -\beta_k + \beta_k)^\T x_i}
    \\
    &\geq
    \cQ
    \sum_k
    \sum_i q_i
    y_i 
    \alpha_k
    \sigma\del{\beta_k^\T x_i}
    -
    \cQ
    \sum_k
    \sum_i q_i
    | \alpha_k|
    \enVert{ \tb_k \tv_k -\beta_k }
    \\
    &\geq
    \cQ
    \ggl
    \sum_i q_i
    -
    \eps
    \cQ
    \sum_k
    \sum_i q_i
    | \alpha_k|
    \\
    &=
    \cQ \sum_i q_i \del{\ggl - \eps}
    > 0.
  \end{align*}
  As an immediate consequence, the signs of all $((a_k,b_k))_{k=1}^r$ never flip (since this
  would require their values to pass through $0$, which would cause $\Phi(W) = -\infty
  < \Phi(W(0))$).
  This implies the preceding lower bound always holds, and since $\Phi(W_0) > -\infty$
  thanks to random initialization, and 
  \[
    \Phi(W)
    =
    \frac 1 4 \sum_{k=1}^r |\alpha_k| \ln \del{ a_k^2 + b_k^2 }
    \leq
    \frac 1 4
    \sum_{k=1}^r |\alpha_k| \ln \|W\|^2
    = \frac 1 2 \ln \|W\|,
  \]
  all conditions of \Cref{fact:abstract:Phi} are satisfied, and the proof is complete
  after applying \Cref{fact:rad}.
\end{proof}

\end{document}